%% file: main.tex
\documentclass[letterpaper]{article} 
\usepackage{aaai24}  
\usepackage{times}  
\usepackage{helvet}  
\usepackage{courier}  
\usepackage[hyphens]{url}  
\usepackage{graphicx} 
\usepackage{natbib}  
\usepackage{caption} 
\usepackage{newfloat}
\usepackage{listings}

\usepackage[utf8]{inputenc} 
\usepackage[T1]{fontenc}    
\usepackage{url}            
\usepackage{booktabs}       
\usepackage{amsfonts}       
\usepackage{nicefrac}       
\usepackage{microtype}      
\usepackage{xcolor}         

\usepackage{amsmath, bm, cleveref, multirow, makecell, mathtools, enumitem, booktabs, subfigure, algorithmic, floatrow, amssymb, thmtools, amsthm, enumitem, pifont}
\newcommand{\cmark}{\ding{51}}
\newcommand{\xmark}{\ding{55}}
\usepackage[ruled,vlined,linesnumbered]{algorithm2e}
\SetKwInput{KwInput}{Input}                
\SetKwInput{KwOutput}{Output}              
\def\name{\textsc{Hot}}
\def\T{{\scriptscriptstyle\mathsf{T}}}
\newtheorem{problem}{Problem}
\newtheorem{definition}{Definition}

\makeatletter
\renewcommand{\maketag@@@}[1]{\hbox{\m@th\normalsize\normalfont#1}}%
\makeatother

\setlist[itemize]{leftmargin=5.5mm}

\newcommand{\hidec}[1]{}

\input{hd}

\DeclareCaptionStyle{ruled}{labelfont=normalfont,labelsep=colon,strut=off} 
\floatstyle{ruled}
\newfloat{listing}{tb}{lst}{}
\floatname{listing}{Listing}

\title{Hierarchical Multi-Marginal Optimal Transport for Network Alignment}
\author {
    Zhichen Zeng\textsuperscript{\rm 1},
    Boxin Du\textsuperscript{\rm 2},
    Si Zhang\textsuperscript{\rm 3},
    Yinglong Xia\textsuperscript{\rm 3},
    Zhining Liu\textsuperscript{\rm 1},
    Hanghang Tong\textsuperscript{\rm 1}
}
\affiliations {
    \textsuperscript{\rm 1}University of Illinois Urbana-Champaign\\
    \textsuperscript{\rm 2}Amazon\\
    \textsuperscript{\rm 3}Meta\\
    zhichenz@illinois.edu, boxin@amazon.com, sizhang@meta.com, yxia@meta.com, liu326@illinois.edu, htong@illinois.edu
}

\usepackage{bibentry}

\begin{document}

\maketitle

\input{0-abs}
\input{1-intro}
\input{2-prob}
\input{3-opt}
\input{4-algo}
\input{5-exp}
\input{6-related}
\input{7-con}

\newpage
\section*{Acknowledgements}
The ZZ and HT are partially supported by NSF (1947135, 
2134079, 
1939725, 
2316233, 
and 2324770
), DARPA (HR001121C0165), DHS (17STQAC00001-07-00), NIFA (2020-67021-32799) and ARO (W911NF2110088).

\bibliography{main}

\appendix
\input{appendix.tex}

\end{document}

%% file: hd.tex
\newcommand{\beqa}{\begin{eqnarray}}
\newcommand{\eeqa}{\end{eqnarray}}
\newcommand{\beq}{\begin{equation}}
\newcommand{\eeq}{\end{equation}}
\newcommand{\ben}{\begin{enumerate}}
\newcommand{\een}{\end{enumerate}}
\newcommand{\bit}{\begin{itemize}}
\newcommand{\eit}{\end{itemize}}
\newcommand{\bi}{\begin{itemize} \item}
\newcommand{\ei}{\end{itemize}}

\newcommand{\begindef}{\begin{Definition} \rm}
\newcommand{\beginexa}{\begin{Example} \rm}
\newcommand{\beginthe}{\begin{Theorem} \rm}
\newcommand{\beginpro}{\begin{Proposition} \rm}
\newcommand{\beginlem}{\begin{Lemma} \rm}
\newcommand{\begincon}{\begin{Conjecture} \rm}
\newcommand{\begincor}{\begin{Corollary} \rm}

\newcommand{\mat}[1]{{\bf #1}}   

\newcommand{\eat}[1]{}
\def\papernumber #1 raised #2 {
\vspace{-#2}
\vbox to 0pt{\hfill\framebox{\bf Paper Number #1}}
\vspace{#2}
}

\newcommand{\ts}{\bm{\mathcal{S}}}
\newcommand{\tc}{\bm{\mathcal{C}}}
\newcommand{\tl}{\bm{\mathcal{L}}}
\newcommand{\G}{\mathcal{G}}
\newcommand{\C}{\mathcal{C}}
\DeclareMathOperator*{\argmax}{arg\,max}
\DeclareMathOperator*{\argmin}{arg\,min}

%% file: 0-abs.tex
\begin{abstract}

Finding node correspondence across networks, namely multi-network alignment, is an essential prerequisite for joint learning on multiple networks. Despite great success in aligning networks in pairs, the literature on multi-network alignment is sparse due to the exponentially growing solution space and lack of high-order discrepancy measures. To fill this gap, we propose a \underline{h}ierarchical multi-marginal \underline{o}ptimal \underline{t}ransport framework named \name\ for multi-network alignment. To handle the large solution space, multiple networks are decomposed into smaller aligned clusters via the fused Gromov-Wasserstein (FGW) barycenter. To depict high-order relationships across multiple networks, the FGW distance is generalized to the multi-marginal setting, based on which networks can be aligned jointly. A fast proximal point method is further developed with guaranteed convergence to a local optimum. Extensive experiments and analysis show that our proposed \name\ achieves significant improvements over the state-of-the-art in both effectiveness and scalability.

\end{abstract}

%% file: 1-intro.tex
\section{INTRODUCTION}\label{sec:intro}
In the era of big data, networks often originate from various domains. Joint learning on multiple networks has shown promising results in various areas including high-order recommendation~\cite{yan2022dissecting}, fraud detection~\cite{du2021new} and fact checking~\cite{liu2021kompare}. A critical steppingstone behind these tasks and many more is the multi-network alignment problem, which aims to find node correspondence across multiple networks.

To date, a multitude of pairwise network alignment methods have been developed based on the consistency principle~\cite{singh2008global,koutra2013big,zhang2016final}, node embedding~\cite{li2019adversarial,chu2019cross,zhang2021balancing}, and optimal transport (OT)~\cite{petric2019got,maretic2022fgot,chen2020graph,zeng2023parrot} with superior performance, but this is not the case for the multi-network setting due to two fundamental challenges. First (\textit{discrepancy measure}), most existing pairwise methods essentially optimize the pairwise discrepancy (e.g., Frobenius norm~\cite{zhang2016final}, contrastive loss~\cite{chu2019cross}, and Wasserstein distance~\cite{maretic2020wasserstein}) between one network and its aligned counterpart, but a similar discrepancy measure for multi-network is lacking. Second (\textit{algorithm}), even equipped with a proper discrepancy measure, an efficient algorithm is demanded to handle the significantly larger solution space of multi-network alignment, compared with its pairwise counterpart.

\noindent\textbf{Contributions.} In this paper, we propose a novel method named \name\ to address the above challenges from the view of multi-marginal optimal transport (MOT)~\cite{pass2015multi}. To jointly measure the discrepancy between multiple networks, the fused Gromov-Wasserstein (FGW) distance is generalized to the multi-marginal setting, whose by-product, the optimal coupling tensor, naturally serves as the alignment between networks. To handle the large solution space, the problem is decomposed into significantly smaller cluster-level and node-level alignment subproblems. Specifically, the cluster-level alignment for multiple networks is obtained based on the FGW barycenter. On top of that, the multi-marginal FGW (MFGW) distance, together with a position-aware cost tensor generated based on the unified random walk with restart (RWR), is adopted for the node-level alignment. To achieve fast solutions, we propose a proximal point method with guaranteed convergence. Extensive experiments show that \name\ outperforms the best competitor by at least 12.0\% on plain networks in terms of high-order Hits@10, with up to 360$\times$ speedup in time complexity and 1000$\times$ reduction in memory cost compared with the non-hierarchical solution.

The rest of the paper is organized as follows. Section~\ref{sec:prob} introduces the preliminaries and problem definitions. Section~\ref{sec:opt} formulates the optimization problem. Section~\ref{sec:algo} presents and analyzes the proposed algorithm. Experiment results are presented in Section~\ref{sec:exp}. We review related work and conclude our paper in Sections~\ref{sec:related} and \ref{sec:con} respectively.

%% file: 2-prob.tex
\section{PROBLEM DEFINITION}\label{sec:prob}

\subsection{Notations}
We use bold uppercase letters for matrices (e.g., $\mathbf{A}$), bold lowercase letters for vectors (e.g., $\mathbf{s}$), calligraphic letters for sets (e.g., $\C$), bold calligraphic letters for tensors (e.g., $\tc$), and lowercase letters for scalars (e.g., $\alpha$).
The element $(i,j)$ of a matrix $\mathbf{A}$ is denoted as $\mathbf{A}(i,j)$, and the element $(i_1,i_2,\ldots,i_K)$ of a tensor $\tc$ is denoted as $\tc(i_1,i_2,\ldots,i_K)$. The transpose of $\mathbf{A}$ is denoted by the superscript $\T$ (e.g., $\mathbf{A}^{\T}$). We use $\Pi(\bm{\mu},\bm{\nu})$ to denote the probabilistic coupling between $\bm{\mu}$ and $\bm{\nu}$, and $\Delta_n=\{\bm{\mu}\in\mathbb{R}_n^+|\sum_{i=1}^n\bm{\mu}(i)=1\}$ to denote a probability simplex with $n$ bins.

For mathematical operations, we use $\odot$ for Hadmard product and $\otimes$ for outer product. We define $\mathcal{P}_{k}(\tc)=\sum_{\{i_1,\dots,i_K\}\setminus\{i_k\}}\tc(i_1,\dots,i_K)$ as the marginal sum of tensor $\tc$ of the $k$-th dimension.

An attributed graph is denoted as $\G=\{\mathbf{A},\mathbf{X}\}$, where $\mathbf{A}$ is the adjacency matrix and $\mathbf{X}$ is the node attribute matrix. We use $n_i$ and $m_i$ to denote the number of nodes and edges in $\G_i$, respectively. Graph indices are indicated by subscripts (e.g., $\G_i$) and cluster indices are indicated by superscripts (e.g., $\C^j$). For a given graph $\G_k$, the $i_k$-th node is denoted as $v_{i_k}$, and the $j$-th cluster is denoted as $\C_k^j$.

Following a common practice in OT-based graph applications~\cite{titouan2019optimal}, an attributed graph can be represented by a probability measure supported on the product space of node attribute and structure, i.e., $\bm{\mu}=\sum_{i=1}^{n}\mathbf{h}(i)\delta_{v_i,\mathbf{X}(v_i)}$, where $\mathbf{h}\in\Delta_{n}$ is a histogram representing the node weight of $v_i\in\G$.
\subsection{Multi-marginal Optimal Transport}
The fused Gromov-Wasserstein (FGW) distance is powerful in processing geometric data by exploring node attributes and graph structure, which is defined as~\cite{titouan2019optimal}:
\begin{definition}\label{def:fgw}
    Fused Gromov-Wasserstein (FGW) distance.\\
    Given two graphs $\G_1=\{\mathbf{A}_1,\mathbf{X}_1\},\G_2=\{\mathbf{A}_2,\mathbf{X}_2\}$ with their probability measures $\bm{\mu}_1,\bm{\mu}_2$ and intra-cost matrices $\mathbf{C}_1,\mathbf{C}_2$ measuring within-graph node relationships, and a cross-cost matrix $\mathbf{C}_{\textup{cross}}$ measuring cross-graph node relationships, the FGW distance $\textup{FGW}_{q,\alpha}(\G_1,\G_2)$ is defined as    

    \begin{equation}\label{eq:fgwd}
    \setlength{\abovedisplayskip}{3pt}
    \setlength{\belowdisplayskip}{3pt}
    \small
        \begin{aligned}
            &\min_{\mathbf{S}\in\Pi(\bm{\mu}_1,\bm{\mu}_2)}(1-\alpha)\sum_{v_1\in\G_1, u_1\in\G_2}\mathbf{C}_{\textup{cross}}^q(v_1,u_1)\mathbf{S}(v_1,u_1)\\
            &+ \alpha \!\!\!\!\!\sum_{v_1,v_2\in\G_1\atop u_1,u_2\in\G_2}\!\!\!\!|\mathbf{C}_1(v_1,v_2)-\mathbf{C}_2(u_1,u_2)|^q\mathbf{S}(v_1,u_1)\mathbf{S}(v_2,u_2)
        \end{aligned}.
    \end{equation}
\end{definition}

The hyperparameter $q$ in Eq.~\eqref{eq:fgwd} is the order of the FGW distance, and we consider $q=2$ for faster computation throughout this paper~\cite{peyre2016gromov}. However, existing FGW distance is only applicable in two-sided OT problems. Based on Definition~\ref{def:fgw}, the FGW distance is generalized to the multi-marginal OT setting as follows: 
\begin{definition}\label{def:mfgw}
Multi-marginal Fused Gromov-Wasserstein (MFGW) distance~\cite{beier2022multi}.\\
Given $K$ graphs $\G_1,\dots,\G_K$ with their probabilistic representations $\bm{\mu}_1,\dots,\bm{\mu}_K$, a cross-cost tensor $\tc\in\mathbb{R}^{n_1\times\dots\times n_K}$ measuring cross-graph node distances based on node attributes, and $K$ intra-cost matrices $\mathbf{C}_k\in\mathbb{R}^{n_k\times n_k},\forall k = 1,\dots,K$ measuring intra-graph node similarity for $\G_k$ based on graph structure. The $q$-MFGW distance $\textup{MFGW}_{q,\alpha}(\G_1,\dots,\G_K)$ is defined as:
\begin{equation}\label{eq:mfgw1}
    \setlength{\belowdisplayskip}{3pt}
    \footnotesize
    \begin{aligned}
        &\min_{\ts\in\Pi(\bm{\mu}_1,\dots,\bm{\mu}_K)}(1-\alpha)\!\!\sum_{v_1,\dots,v_K}\!\!\tc(v_1\!,\!\dots\!,\!v_K)^q\ts(v_1\!,\!\dots\!,\!v_K)+\\
        &\alpha\!\!\!\!\!\sum_{
        \substack{1\leq j,k \leq K\\ v_1,\dots,v_K\\ v_1'\dots,v_K'}}\!\!\!\!\!|\mathbf{C}_j(v_j,\!v_j')\!-\!\mathbf{C}_k(v_k,\!v_k')|^q\ts(v_1\!,\!\dots\!,\!v_K)\ts(v_1'\!,\!\dots\!,\!v_K')
    \end{aligned}
\end{equation}
\end{definition}
Intuitively, the first summation is the Wasserstein term measuring the joint distance for $K$ graphs in terms of node attributes. The second summation is the Gromov-Wasserstein term measuring the structural difference among all node pairs in $K$ graphs weighted by the optimal coupling tensor $\ts$.

\subsection{Hierarchical Multi-network Alignment}
\begin{problem}\label{prob:hierarchical}
Hierarchical multi-network alignment.\label{prob:mna}\\
\textbf{Given:} (1) $K$ attributed networks $\mathcal{G}_i=\{\mathbf{A}_i,\mathbf{X}_i\}$, and (2) a set of anchor node sets $\mathcal{L}$ indicating which nodes are aligned a priori.\\
\textbf{Output:} (1) cluster-level alignment sets $\C^j=\bigcup_{i=1}^K\C_i^j$ for $j=1,\ldots,M$, where $M$ is the number of clusters and $\C_i^j$ is the set of nodes from $\G_i$ that are clustered to the $j$-th cluster, and (2) node-level alignment tensors $\ts^j$ for $\C^j$, whose entry indicates how likely nodes are aligned.
\end{problem}
An illustrative example is shown in Figure~\ref{fig:mna}. Given the anchor node set $\mathcal{L}$, the $1^{st}$ cluster-level alignment $\C^1$ (red circle in the middle figure) consists of clusters $\C_1^1$, $\C_2^1$ and $\C_3^1$, and corresponding $\ts^1$ indicates alignments among nodes in $\C_1^1$, $\C_2^1$ and $\C_3^1$. Note that Problem~\ref{prob:hierarchical} is a generalized version of single-level pairwise network alignment. For example, when $K=2$, the problem degenerates to the hierarchical pairwise alignment problem~\cite{xu2019scalable,zhang2019multilevel}. When $M=1$, the problem degenerates to the single-level multi-network alignment problem~\cite{chu2019cross}.

\begin{figure}
    \centering
    \includegraphics[width = \linewidth]{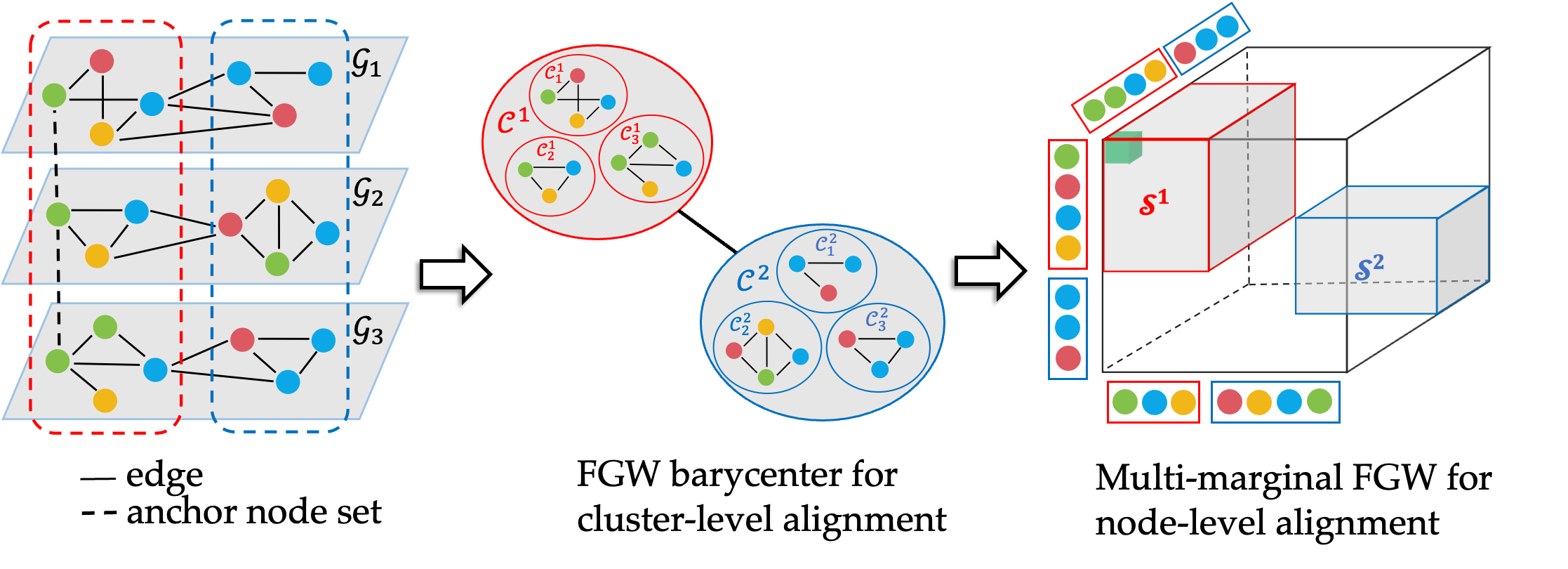}
    \caption{An overview of \name. Left: three input networks, where three green nodes connected by the black dash line form an anchor node set. Middle: FGW barycenter co-clusters three graphs into two clusters. Right: the node alignment tensor with blocks $\ts^1$ for cluster $\C^1$ and $\ts^2$ for cluster $\C^2$.}
    \label{fig:mna}
\end{figure}

%% file: 3-opt.tex
\section{OPTIMIZATION FORMULATION}\label{sec:opt}
In this section, we present our hierarchical MOT-based multi-network alignment framework. A position-aware cost tensor is first developed to depict high-order relationships across networks. Then the multi-network alignment problem is formulated as a hierarchical MOT problem, including cluster-level alignmend based on FGW barycenter and node-level alignment based on MFGW distance.

\subsection{Position-Aware Cost Tensor}\label{sec:cost}
Modeling node relationship across multiple networks is essential for multi-network alignment. Consistency-based methods~\cite{du2021sylvester,li2021scalable,zhang2016final} model node relationships by the Kronecker product graph, the size of which becomes intractable for large networks. Embedding-based methods~\cite{heimann2018regal,zhang2020nettrans,zhang2021balancing} generate embedding spaces for different network pairs but suffer from the space disparity issue.

To overcome the above limitations, we adopt the unified RWR to generate position-aware node embeddings in a unified space~\cite{yan2021bright,yan2024pacer}. The idea is to treat nodes in an anchor node set as one identical landmark in the embedding space and construct a unified space by encoding positional information with respect to (w.r.t.) same landmarks. Formally speaking, given the $p$-th anchor node set $\{l_{1_p},\ldots,l_{K_p}\}\in\mathcal{L}$ where $l_{i_p}$ is the anchor node from $\mathcal{G}_i$, the RWR score vector $\mathbf{r}_{i_p}\in\mathbb{R}^{n_i}$ depicting the relative positions of nodes from $\G_i$ w.r.t. $l_{i_p}$ is computed by~\cite{tong2006fast}
\begin{equation}\label{eq:rwr}
    \setlength{\abovedisplayskip}{3pt}
    \setlength{\belowdisplayskip}{3pt}
    \mathbf{r}_{i_p} = (1-\beta)\mathbf{W}_i\mathbf{r}_{i_p}+\beta\mathbf{e}_{i_p},
\end{equation}
where $\beta$ is the restart probability, $\mathbf{W}_i=(\mathbf{D}_i^{-1}\mathbf{A}_i)^{\T}$ is the transpose of the row normalized matrix of $\mathbf{A}_i$, and $\mathbf{e}_{i_p}$ is an $n_i$-dimensional one-hot vector with $\mathbf{e}_{i_p}(l_{i_p})=1$. The final positional embedding is the concatenation of the RWR scores w.r.t. different anchor node sets in $\mathcal{L}$, i.e., $\mathbf{R}_{i}=[\mat r_{i_1}\|\ldots\|\mat r_{i_{|\mathcal{L}|}}]\in\mathbb{R}^{n_i\times |\mathcal{L}|}$.

When node attributes are available, we use the concatenation of node attribute and positional embedding as the node embedding, i.e., $\mathbf{Z}_i = \left[\mathbf{X}_i\|\mathbf{R}_i\right]$ for $\G_i$. Otherwise, we simply use $\mathbf{R}_i$ as the node embedding, i.e., $\mathbf{Z}_i = \mathbf{R}_i$. Given a set of node embeddings $\{\mathbf{Z}_1,\ldots,\mathbf{Z}_K\}$, the position-aware cost tensor $\tc$ is computed by the total sum of all pairwise node distances as follows~\cite{alaux2018unsupervised}:
\begin{equation}\label{eq:costtensor}
    \setlength{\abovedisplayskip}{3pt}
    \setlength{\belowdisplayskip}{3pt}
    \tc(v_1,\dots,v_K) = \sum_{1\leq j,k\leq K}\|\mathbf{Z}_j(v_j)-\mathbf{Z}_k(v_k)\|_2.
\end{equation}

\subsection{FGW-based Cluster-level Alignment}\label{sec:cluster}
Hierarchical structures are ubiquitous in real-world networks, and exploring such cluster structures can benefit the multi-network alignment task in both effectiveness and scalability~\cite{zhang2019multilevel,jing2023sterling,liu2019g}. For example, as shown in Figure~\ref{fig:mna}, if cluster-level alignments are known, we can dramatically shrink the solution space by only considering nodes in the aligned clusters for node-level alignments. To obtain high-quality cluster-level alignments, we follow a similar approach as~\cite{xu2019scalable} based on the FGW barycenter~\cite{titouan2019optimal}.

Given $K$ networks $\mathcal{G}_i=\{\mathbf{A}_i,\mathbf{X}_i\}$ and their probability measures $\bm{\mu}_i$, the FGW barycenter $\mathcal{G}_b=\{\mathbf{A}_b,\mathbf{X}_b\}$ serves as a consensus graph that is close to all given graphs in terms of the FGW distance. Regarding each node in $\mathcal{G}_b$ as the barycenter of one cluster, nodes transported to the same barycenter form a cluster-level alignment. Specifically, we adopt the $L_2$ norm between node attributes as the cross-cost matrices, i.e., $\mathbf{C}_{\text{cross}_i}(v,u)=\|\mathbf{X}_i(v)-\mathbf{X}_b(u)\|_2,\forall v\in\G_i, u\in\G_b$, to depict node relationships between $\mathcal{G}_i$ and $\mathcal{G}_b$. The FGW-based cluster-level alignment problem is formulated as
\begin{align}\label{eq:fgwb}
    \setlength{\abovedisplayskip}{3pt}
    \setlength{\belowdisplayskip}{3pt}
    &\mathop{\argmin}_{\mathbf{A}_b,\mathbf{X}_b}\sum_{i=1}^K \text{FGW}_{2,\alpha}(\mathbf{C}_{\text{cross}_i},\mathbf{A}_i,\mathbf{A}_b,\bm{\mu}_i,\bm{\mu}_b).
\end{align}
By exploiting the OT coupling $\mathbf{S}_i$ between $\G_i$ and barycenter $\G_b$, nodes $v_i\in\G_i$ are determinisitically assigned to the barycenter node $b_j\in\G_b$ such that $b_j=\argmax_{b\in\G_b}\mathbf{S}_i(v_i,b)$. Note that these barycenter nodes $b_j$ serves as "references" connecting nodes $v_i$ in cluster $\C_i^j$ in different graph $\G_i$, hence providing a cluster-level alignment $\C^j=\bigcup_{i=1}^K\C_i^j$. An illustrative example is given by the middle subfigure of Figure~\ref{fig:mna}.

\subsection{MFGW-based Node-level Alignment}\label{sec:node}
The MFGW distance in Definition~\ref{def:mfgw} provides a joint distance measure for multiple networks given their attributes and structure, and the optimal coupling $\ts$, as a by-product of the MFGW distance, indicates the high-order node alignments across networks.

To make the computation more tractable, we first propose a tensor form MFGW distance as follows
\begin{restatable}{proposition}{mfgw}
    The MFGW distance in Eq.~\eqref{eq:mfgw1} with $q=2$ can be formulated into a tensor form as:
    \begin{equation}\label{eq:mfgw2}
        \min_{\ts\in\Pi(\bm{\mu}_1,\dots,\bm{\mu}_K)}\langle(1-\alpha)\tc + \alpha\tl,\ts\rangle,
    \end{equation}
    where $\tl(v_1,...,v_K) = (K-1)\sum_{j=1}^K\mathbf{C}_j(v_j,\cdot)^2 \mathcal{P}_j(\ts) - 2\sum_{1\leq j<k\leq K}\mathbf{C}_j(v_j,\cdot)\mathcal{P}_{j,k}(\ts)\mathbf{C}_k(v_k,\cdot)^\T$.
 \end{restatable}

Directly applying the MFGW on node alignments still leads to intractable time and space complexities. To overcome this issue, we achieve an exponential reduction in both complexities by only considering node alignments inside the aligned clusters $\mathcal{C}^j$, which decomposes the original problem of size $\mathcal{O}(n^K)$ into $M$ independent in-cluster node-level alignment subproblems, each with size $\mathcal{O}\left(\left(\frac{n}{M}\right)^K\right)$. Following a common practice~\cite{titouan2019optimal}, we represent clusters $\mathcal{C}_i^j\in\mathcal{C}^j$ as discrete uniform distributions $\bm{\mu}_i^j=\bm{1}/|\mathcal{C}_i^j|$ supported on its nodes. Together with the position-aware cost tensor $\tc^j$ in Eq.~\eqref{eq:costtensor} and the intra-cluster adjacency matrices $\mathbf{A}^j_1,\dots,\mathbf{A}^j_2$ describing intra-cluster node connectivity, the node-level alignment subproblem is formulated as the following MFGW problem:
\begin{equation}\label{eq:node}
    \setlength{\abovedisplayskip}{1pt}
    \setlength{\belowdisplayskip}{3pt}
    \small
    \ts^j=\mathop{\argmin}_{\ts\in\Pi(\bm{\mu}_1^j,\ldots\bm{\mu}_K^j)}\langle(1-\alpha)\tc^j + \alpha\tl^j,\ts\rangle, \forall j=1,\dots,M,
\end{equation}
where $\ts^j$ is the node-level alignment tensor for $\mathcal{C}^j$.

%% file: 4-algo.tex
\section{ALGORITHM AND ANALYSIS}\label{sec:algo}
In this section, we present and analyze our optimization algorithm \name. We first adopt the block coordinate descent (BCD) method to solve the FGW-based cluster-level alignment. Afterward, the MFGW-based node-level alignment is solved by the proximal point method to a local optimum. Relevant analyses of the proposed \name\ are carried out thereafter.

\subsection{Optimization Algorithm}\label{sec:opt algo}
\paragraph{FGW-based cluster-level alignment} in Eq.~\eqref{eq:fgwb} is a non-convex multivariate optimization problem and can be efficiently solved by the BCD algorithm~\cite{ferradans2014regularized}. Specifically, the objective is minimized w.r.t. $\mathbf{S}_i$, $\mathbf{A}_b$ and $\mathbf{X}_b$ iteratively. For the $t$-th iteration, the minimization w.r.t. three variables are calculated as follows.

First, fixing $\mathbf{A}_b$ and $\mathbf{X}_b$, the optimization w.r.t. $\mathbf{S}_i$ is formulated as
\begin{equation}\label{eq:opt_s}
    \setlength{\abovedisplayskip}{1pt}
    \setlength{\belowdisplayskip}{2pt}
    \begin{aligned}
        \mathbf{S}_i^{(t+1)}&=\sum_{j=1}^K \min_{\mathbf{S}_j\in\Pi(\bm{\mu}_j,\bm{\mu}_b)} \!\langle(1-\alpha)\mathbf{C}_{\text{cross}_j}^{(t)}\!+\!\alpha\mathbf{L}^{(t)}_j,\mathbf{S}_j\rangle\\
        &=\min_{\mathbf{S}_i\in\Pi(\bm{\mu}_i,\bm{\mu}_b)} \langle(1-\alpha)\mathbf{C}_{\text{cross}_i}^{(t)}+\alpha\mathbf{L}^{(t)}_i,\mathbf{S}_i\rangle.
    \end{aligned}
\end{equation}
The last equation is due to the fact that $\mathbf{S}_i^{(t)}$ are decoupled from each other, so it is equivalent to minimizing $K$ FGW distances independently. Note that the optimization problem in Eq.~\eqref{eq:opt_s} is a special case (i.e., two-sided OT setting) of the MFGW problem in Definition~\ref{def:mfgw}, and can be efficiently solved by the proximal point method introduced later in this section.

Second, fixing $\mathbf{S}_i$ and $\mathbf{X}_b$, the optimal value for the adjacency matrix $\mathbf{A}_b$ of $\G_b$ can be computed by the first-order optimality condition as~\cite{peyre2016gromov}
\begin{equation}\label{eq:opt_a}
    \setlength{\abovedisplayskip}{1pt}
    \setlength{\belowdisplayskip}{3pt}
    \mathbf{A}_b^{(t+1)}=\frac{\mathbf{1}_{M\times M}}{\bm{\mu}_b\bm{\mu}_b^{\T}}\sum_{i=1}^{K}\left(\mathbf{S}_i^{(t+1)^{\T}}\mathbf{A}_i\mathbf{S}_i^{(t+1)}\right).
\end{equation}

Third, fixing $\mathbf{S}_i$ and $\mathbf{A}_b$, the objective function is quadratic w.r.t. the node attribute matrix $\mathbf{X}_b$, whose optimal value can be efficiently computed as~\cite{cuturi2014fast}
\begin{equation}\label{eq:opt_x}        
    \mathbf{X}^{(t+1)}_b=\sum_{i=1}^{K}\left(\text{diag}\left(\frac{\mathbf{1}_{M}}{\bm{\mu}_b}\right)\mathbf{S}_i^{(t+1)^{\T}}\mathbf{X}_i\right).
\end{equation}

By iteratively applying Eqs.~\eqref{eq:opt_s}-\eqref{eq:opt_x}, the algorithm converges to the local optimal barycenter~\cite{titouan2019optimal}. 

\paragraph{MFGW-based node-level alignment.} 
In order to handle the non-convex objective function in Eq.~\eqref{eq:node}, we generalize the proximal point method~\cite{xu2019gromov} to the multi-marginal setting with guaranteed convergence to a local optimum. The key idea is to decompose the non-convex problem into a series of convex subproblems regularized by the proximal operator. We adopt the KL divergence as the proximal operator, i.e., $\text{KL}(\ts\|\ts^{(t)})$, to regularize the distance between two successive solutions, and the resulting problem corresponds to a regularized MOT problem as follows:
\begin{equation}\label{eq:proximal}
    \begin{aligned}
        &\ts^{(t+1)}=\!\!\!\!\!\mathop{\arg\min}\limits_{\ts\in\Pi(\bm{\mu}_1^j,\ldots,\bm{\mu}_K^j)}\!\!\!\!\!\langle(1-\alpha)\tc \!+\! \alpha\tl^{(t)},\ts\rangle \!+\!\lambda \text{KL}(\ts\|\ts^{(t)})\\
        &=\!\!\!\mathop{\arg\min}\limits_{\ts\in\Pi(\bm{\mu}_1^j,\ldots,\bm{\mu}_K^j)}\!\!\!\langle\bm{\mathcal{Q}}^{(t)},\ts\rangle +\lambda \langle\ts,\log\ts\rangle,
    \end{aligned}
\end{equation}
where $\bm{\mathcal{Q}}^{(t)}\!=\!(1-\alpha)\tc + \alpha\tl^{(t)}\!
-\!\lambda\log\ts^{(t)}$ is fixed when optimizing $\ts$; hence, the resulting problem corresponds to an entropic regularized OT problem with modified cost tensor $\bm{\mathcal{Q}}^{(t)}$ and can be efficiently solved by the Sinkhorn algorithm~\cite{cuturi2013sinkhorn}.

Specifically, with initial scaling vectors $\mathbf{u}_i^{(0)}$, the algorithm iteratively updates scaling vectors by
\begin{equation}\label{eq:sinkhorn1}
    \setlength{\abovedisplayskip}{2pt}
    \setlength{\belowdisplayskip}{3pt}
    \mathbf{u}_i^{(l+1)} =  \frac{\mathbf{u}_i^{(l)}\odot\bm{\mu}_i^j}{\mathcal{P}_i\left[\text{exp}(-\frac{\bm{\mathcal{Q}}^{(t)}}{\lambda})\odot\bigotimes_{i=1}^K \mathbf{u}_i^{(l)}\right]}.
\end{equation}
After $L$ inner iterations of Eq.~\eqref{eq:sinkhorn1}, the final solution $\ts^{(t+1)}$ can be computed as
\begin{equation}\label{eq:sinkhorn2}
    \setlength{\abovedisplayskip}{2pt}
    \setlength{\belowdisplayskip}{2pt}
    \ts^{(t+1)}=\text{exp}(-\frac{\bm{\mathcal{Q}}^{(t)}}{\lambda})\odot\bigotimes_{i=1}^K \mathbf{u}_i^{(L)}.
\end{equation}
As we will show in Section~\ref{sec:ana}, by iteratively applying Eqs.~\eqref{eq:proximal}-\eqref{eq:sinkhorn2}, the solution sequence given by the proposed proximal point method converges to a local optimum of the MFGW distance. 

\subsection{Theoretical Analysis}\label{sec:ana}
Without loss of generality, we assume that networks share a comparable size, each with $\mathcal{O}(n)$ nodes and $\mathcal{O}(m)$ edges. For brevity, we denote the average cluster size as $\overline{n}=\frac{n}{M}$. 

\paragraph{Complexity analysis.} We first provide a complexity analysis of the proposed \name\ as follows
\begin{restatable}{proposition}{complexity}
    With $K$ graphs, $M$ clusters, and $T$ proximal point iterations, the space complexity of \name\ is $\mathcal{O}(M\overline{n}^K)$, and the time complexity is $\mathcal{O}(TKM(n^2+K\overline{n}^K))$.
\end{restatable}

For space complexity, the overall $\mathcal{O}(M\overline{n}^K)$ achieves an exponential reduction of space in terms of the number of graphs $K$ compared to $\mathcal{O}(n^K)$ given by the straightforward method, which finds the full node-level alignment tensor without cluster-level alignment. For time complexity, the first term $\mathcal{O}(TKMn^2)$ corresponds to the cluster-level alignment, and the second term $\mathcal{O}(TK^2M\overline{n}^K)$ accounts for the node-level alignment. For cases where $\mathcal{O}(\overline{n})<\mathcal{O}\left(\sqrt[K]{n^2/K}\right)$, the time complexity is determined by the cluster-level alignment and can be approximated by $\mathcal{O}(TKMn^2)$, which is quadratic w.r.t. $n$ and linear w.r.t. $K$. Otherwise, the time complexity mostly lies in the node-level alignment and can be approximated by $\mathcal{O}(TK^2M\overline{n}^K)$, which is polynomial w.r.t. $\overline{n}$ and exponential w.r.t. $K$. Since that $\overline{n}=\frac{n}{M}$, we achieve an exponential reduction of time in terms of the number of graphs $K$ compared to $\mathcal{O}(TK^2n^K)$ of the straightforward method.

\paragraph{Optimality and convergence.} First, for the position-aware cost tensor in Eqs.~\eqref{eq:rwr} and \eqref{eq:costtensor}, the computation has guaranteed convergence via the fixed point method as the eigenvalues of $\mathbf{W}_i$ lie in [-1,1]. Second, for the FGW barycenter computation, the solution sequence given by the BCD method converges to a stationary point~\cite{titouan2019optimal}. Third, for the MFGW computation, we have the following proposition stating that the proximal point method converges to a local optimum of the MFGW problem.
\begin{restatable}{proposition}{converge}\label{prop:converge}
    The solution sequence $\ts^{(t)}$ given by the proximal point method converges to a stationary point of the MFGW problem in Definition~\ref{def:mfgw}.
\end{restatable}
The general idea is to show the regularized objective is an upper bound of the original objective and further take advantage of the convergence theorem of the successive upper-bound minimization method~\cite{razaviyayn2013unified,xu2019gromov}. Therefore, the proposed \name\ is guaranteed to converge to the local optimum.

\paragraph{Connection with pairwise FGW distance.} Besides, we reveal the close connection between FGW and MFGW distance. When adopting the square loss, i.e., $q=2$, the MFGW distance is lower bounded by the sum of FGW distances between all possible pairs. In other words, the joint distance between multiple networks is likely to be underestimated by the pairwise FGW distance.
\begin{restatable}{theorem}{bound}
    Given $K$ graphs $\G_1,\G_2,\dots,\G_K$, the MFGW distance is lower bounded by the sum of all pairwise FGW distances, that is:
    \begin{equation*}
        \sum_{1\leq j<k\leq K}\textup{FGW}_{2,\alpha}(\G_j,\G_k)\leq \textup{MFGW}_{2,\alpha}(\G_1,\G_2,\dots,\G_K).
    \end{equation*}
\end{restatable}

%% file: 5-exp.tex
\section{EXPERIMENTS}\label{sec:exp}
We evaluate the proposed \name\ from the following aspects:
\begin{itemize}[itemsep=1pt,topsep=0pt,parsep=0pt]
    \item Q1. How effective is \name\ (Section~\ref{sec:exp-eff})?
    \item Q2. How scalable is \name\ (Section~\ref{sec:exp-sca})?
    \item Q3. How is the convergence of \name\ (Section~\ref{sec:exp-conv})?
    \item Q3. How robust is \name\ to hyperparameters (Section~\ref{sec:exp-hyper})?
\end{itemize}

\paragraph{Datasets.} Our method is evaluated on both plain networks, including Douban, ER and DBLP, and attributed networks, including ACM(A) and DBLP(A). 
To mitigate the effect of data split, we randomly split the datatsets into 10 folds, using 1 fold (i.e., 10\%) for training and the rest 9 folds for testing. We report the mean and standard deviation of the alignment results with different training/test splits\footnote{Code and datasets are available at \url{https://github.com/zhichenz98/HOT-AAAI24}}.

\paragraph{Baseline methods.} The proposed \name\ is compared with a variety of baseline methods, including (1) consistency-based methods: IsoRank~\cite{singh2008global}, FINAL~\cite{zhang2016final}, MOANA~\cite{zhang2019multilevel}, and SYTE~\cite{du2021sylvester}, (2) embedding-based methods: CrossMNA~\cite{chu2019cross}, NetTrans~\cite{zhang2020nettrans}, NeXtAlign~\cite{zhang2021balancing}, and Grad-Align~\cite{park2022grad}, and (3) OT-based methods: GW~\cite{memoli2011gromov}, FGW~\cite{titouan2019optimal}, Low-rank OT (LOT)~\cite{scetbon2021low}, S-GWL~\cite{xu2019scalable}, and WAlign~\cite{gao2021unsupervised}. For pairwise alignment methods, we run them on each pair of networks and integrate the alignment matrices by multiplication (e.g., for networks $\mathcal{G}_1,\mathcal{G}_2,\mathcal{G}_3$, the alignment tensor is obtained by $\ts(x,y,z)=\mathbf{S}_{\mathcal{G}_1,\mathcal{G}_2}(x,y)\mathbf{S}_{\mathcal{G}_1,\mathcal{G}_3}(x,z)\mathbf{S}_{\mathcal{G}_2,\mathcal{G}_3}(y,z)$). 

\paragraph{Parameter settings.} In our experiments, we adopt a consistent parameter setting with $\lambda\!=\!10^{-3}$, $\alpha\!=\!0.5$, and $\beta\!=\!0.15$. For number of clusters, we set $M=\lceil\frac{n}{50}\rceil$ for all datasets.

\paragraph{Metrics.} We evaluate the effectiveness in terms of pairwise Hits@{\em K} (PH@{\em K}), high-order Hits@{\em K} (HH@{\em K})~\cite{du2021sylvester} and Mean Reciprocal Rate (MRR). Given a test node $x_1\in\mathcal{G}_1$, if any corresponded node $x_i\in\mathcal{G}_i$ exists in the top-{\em K} most similar node sets, it is regarded as a pairwise hit. Only if the whole corresponded node set $\{x_i\}$ appears in the top-{\em K} most similar node sets, it is regarded as a high-order hit. For a test dataset with $n$ node sets, the Hits@{\em K} is computed by $\text{Hits@{\em K}}=\frac{\# \text{ of hits}}{n}$, and MRR is computed by the average of the inverse of high-order alignment ranking $\text{MRR}=\frac{1}{n}\sum_{i=1}^n\frac{1}{\text{rank}(\{x_i\})}$

\subsection{Effectiveness Results}\label{sec:exp-eff}
The alignment results on plain and attributed networks are shown in Figures~\ref{fig:plain-1} and \ref{fig:attributed-1} respectively.
For better visualization, we use "$\times$" for consistency-based, "$\bullet$" for embedding-based and "$\star$" for OT-based methods.

\begin{figure*}[h]
    \centering
    \includegraphics[width = \textwidth,trim = 5 10 10 1,clip]{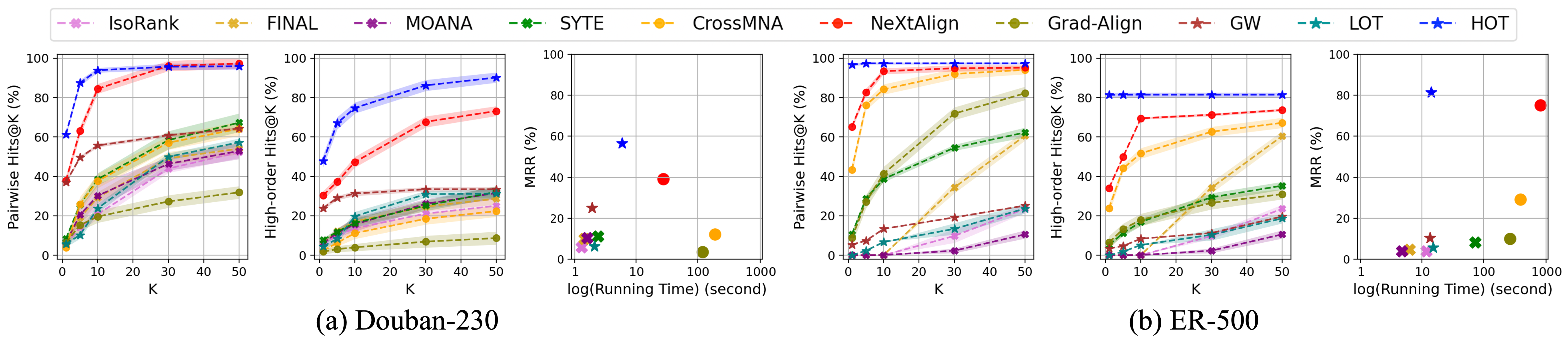}
    \caption{Alignment results on plain networks: (a) Douban-230; (b) ER-500.}\label{fig:plain-1}
\end{figure*}

\begin{figure*}[t]
    \centering
    \includegraphics[width = \textwidth,trim = 5 10 10 1,clip]{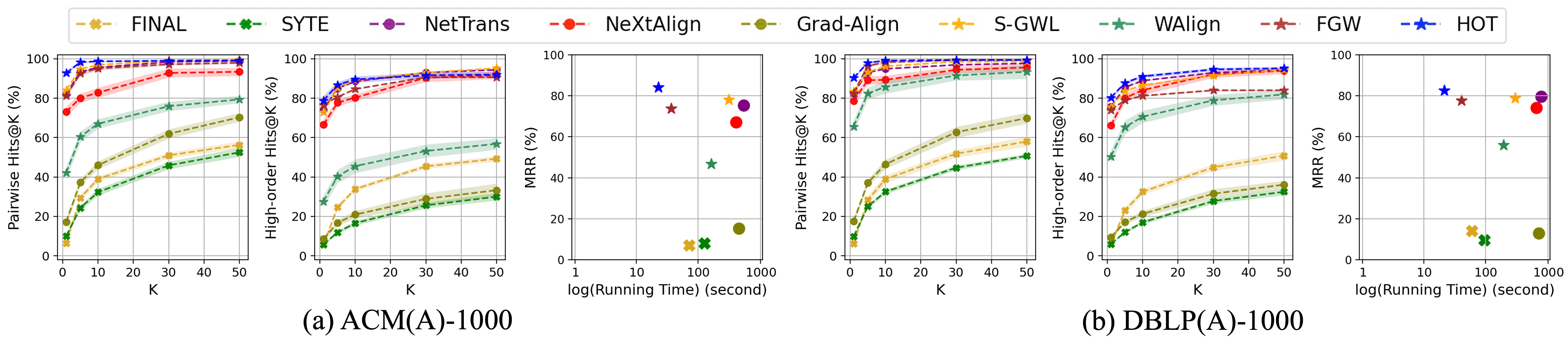}
    \caption{Alignment results on attributed networks: (a) ACM(A)-1000; (b) DBLP(A)-1000.}\label{fig:attributed-1}
\end{figure*}

It is shown that \name\ outperforms all baselines in most cases with more significant outperformance on
\begin{itemize}
    \item plain networks than attributed networks, thanks to the position-aware cost tensor and the MFGW distance addressing the topological relationship across multiple networks jointly, which is particularly important to plain networks where node attributes are unavailable.
    \item high-order metrics than pairwise metrics, as the proposed \name\ aligns networks jointly, whereas many baselines align networks in pairs, resulting in incompatible alignment scores and node embeddings in disparate spaces.
    \item harder metrics (e.g., Hits@1) than softer metrics (e.g., Hits@50). This is due to the exponential term in Eq.~\eqref{eq:opt_s} that provides more deterministic and noise-reduced alignments~\cite{mena2018learning, zeng2023parrot}, compared with the uncertain alignments given by many baselines.
\end{itemize}

Compared with consistency-based methods, \name\ outperforms the best competitor by at least 55.3\% in PH@10, 51.9\% in HH@10, and 42.9\% in MRR on plain networks. On attributed networks, \name\ surpasses  the best competitor\footnote{Slight deviations from results in~\cite{du2021sylvester} may exist due to different ways to categorize node attributes.} at least 57.3\% in PH@10, 53.3\% in HH@10, and 68.2\% in MRR. The limited performance of consistency-based methods owes to the fact that the consistency principle only enforces the consistency between aligned node pairs~\cite{zhang2016final} and fail to model the high-order node relationships jointly. Besides, \name\ achieves comparable running time as consistency-based methods when aligning small networks (e.g., Douban-230 and ER-500) and faster speed when aligning large networks (e.g., ACM(A)-1000 and DBLP(A)-1000).

In comparison to embedding-based methods, \name\ outperforms the best competitor by at least 4.0\% in PH@10, 12.0\% in HH@10, and 1.2\% in MRR on plain networks. On attributed networks, \name\ surpasses the best competitor by at least 3.1\% in PH@10, 1.0\% in HH@10, and 2.8\% in MRR. While the best competitor \textsc{NeXtAlign}~\cite{zhang2021balancing} achieves comparable PH@{\em K}, \name\ significantly outperforms it in HH@{\em K} and MRR. This is because embedding-based methods basically optimize a ranking-based loss addressing pairwise distances, while the joint high-order relationships are largely ignored.

\begin{figure}[t]
\centering
\subfigure[]{\includegraphics[width = 0.48\linewidth,trim = 12 1 15 1,clip]{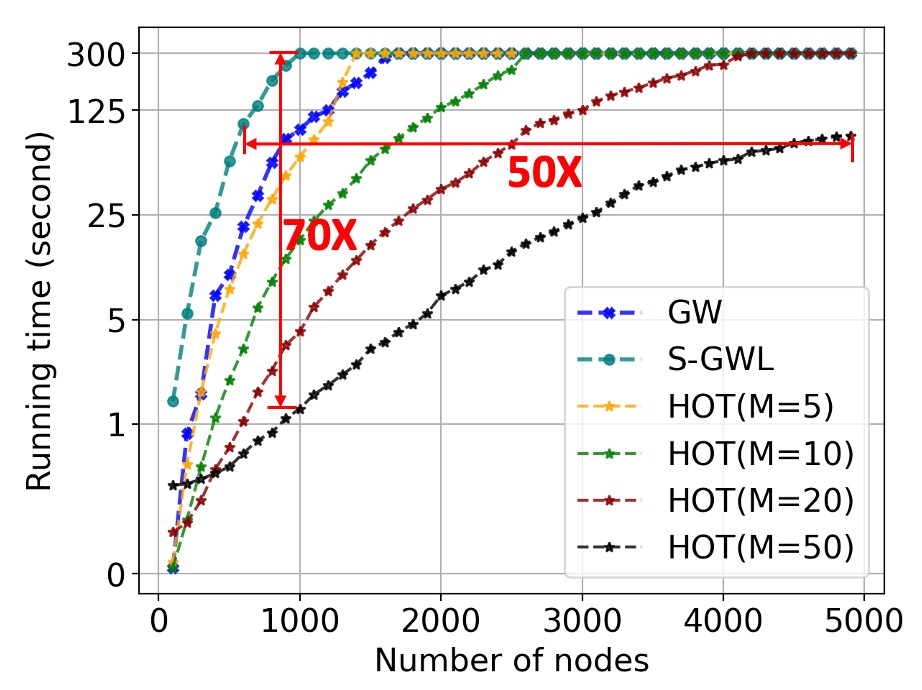}}
\subfigure[]{\includegraphics[width = 0.48\linewidth,trim = 12 1 15 1,clip]{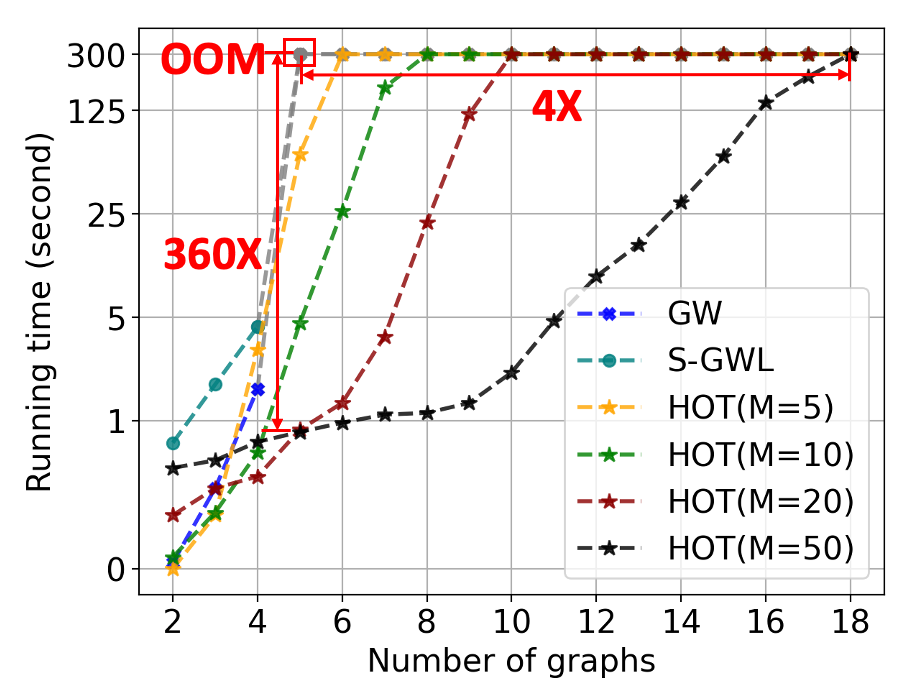}}
\caption{Experiments on time complexity w.r.t. (a) number of nodes, and (b) number of graphs: grey points indicate out-of-memory. Note that the running time is in the log scale.}\label{fig:time}
\end{figure}
For OT-based methods, empirical evaluation shows that \name\ outperforms the best competitor by at least 35.4\% in PH@10, 15.8\% in HH@10, and 6.3\% in MRR. On attributed networks, \name\ outperforms the best OT-based competitor by at least 0.2\% in PH@10, 0.4\% in HH@10, and 1.4\% in MRR. Although LOT adopts a similar idea as \name\ that explores the low rank structure in the pairwise alignment matrix, i.e., cluster structure in graphs, the low rank structure is inconsistent for different network pairs, hence may fail to be generalized to the multi-network setting.

\subsection{Scalability Results}\label{sec:exp-sca}
We study the scalability of the proposed \name. In general, we evaluate the scalability w.r.t. the number of nodes $n$ by aligning three ER graphs with different sizes, and the scalability w.r.t. the number of graphs $K$ by aligning multiple ER graphs with 100 nodes. The time complexity and space complexity results are shown in Figure~\ref{fig:time} and Figure~\ref{fig:space} respectively. We terminate the program if it can not finish in 300 seconds or exceeds the memory capacity (8GB).

\begin{figure}[t]
\centering
\subfigure[]{\includegraphics[width = 0.48\linewidth,trim = 2 2 12 0,clip]{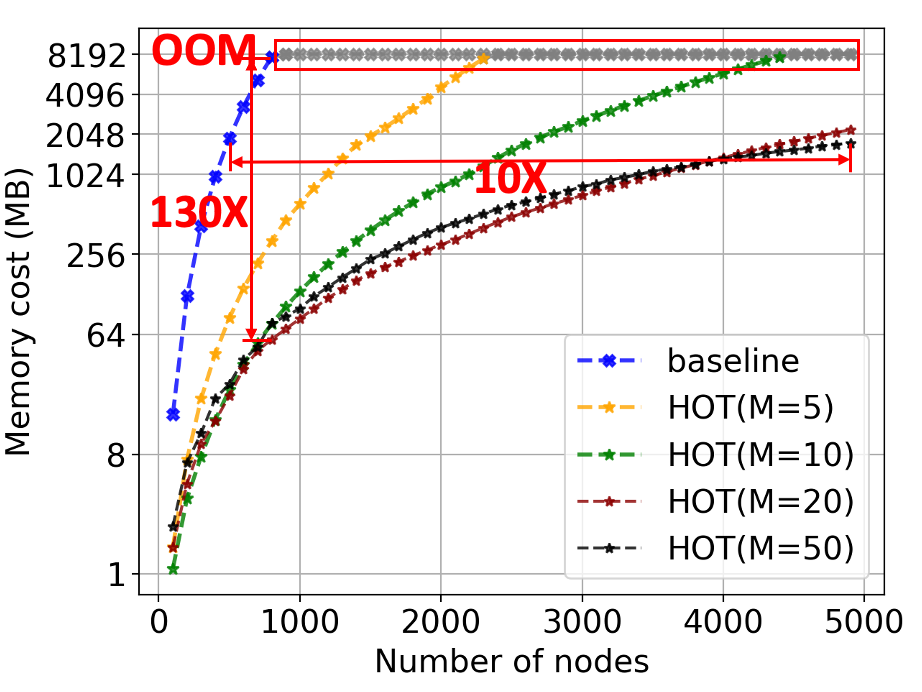}}
\subfigure[]{\includegraphics[width = 0.48\linewidth,trim = 2 2 15 0,clip]{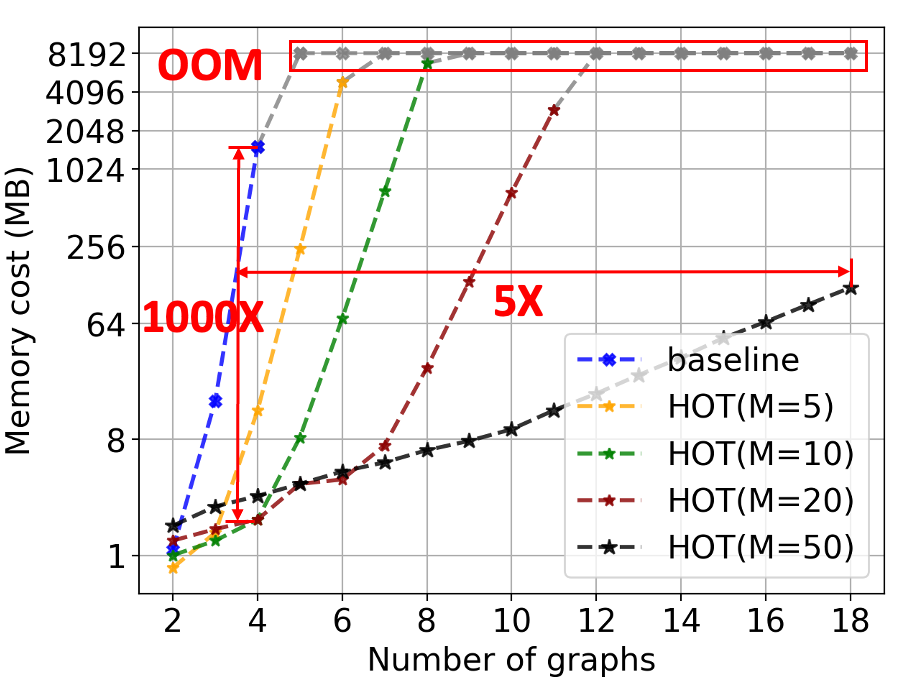}}
\caption{Experiments on space complexity w.r.t. (a) number of nodes, and (b) number of graphs: grey points indicate out-of-memory. Note that the memory cost is in the log scale.}\label{fig:space}
\end{figure}
For time complexity, when the number of nodes $n$ scales up (Figure~\ref{fig:time}(a)), \name\ achieves up to 70$\times$ faster speed and 50$\times$ scale up in $n$ compared with baselines. When the number of graphs $K$ scales up (Figure~\ref{fig:time}(b)), \name\ achieves up to 360$\times$ faster speed and 4$\times$ scale up in $K$ compared with baselines. Such substantial outperformance can be attributed to the hierarchical nature of \name (i.e., the block-diagonal structure of $\ts$). In contrast, other baselines with dense alignment tensor are out-of-memory (OOM) when $K\geq 5$.

For space complexity, when the number of nodes $n$ scales up (Figure~\ref{fig:space}(a)), \name\ only consumes $\frac{1}{130}\times$ memory and scales up 10$\times$ in $n$ compared with baselines. When the number of graphs $K$ scales up (Figure~\ref{fig:space}(b)), \name\ only consumes $\frac{1}{1000}\times$ memory and scales up 5$\times$ in $K$ compared with baselines. Similarly, this is due to the utilization of block-diagonal property of $\ts$ achieving an exponential reduction in time and space complexities.

Besides, comparing \name\ with different cluster number $M$ in both Figures~\ref{fig:time} and~\ref{fig:space}. When the cluster size $\overline{n}$ is small (e.g., left-most points with the number of nodes = 100), larger $M$ results in higher time and space complexities as the cluster-level calculation is the dominant factor under such condition. However, when more nodes or networks are involved, both complexities are mostly dominated by the node-level alignment, where a larger $M$ is preferred for better scalability.

\begin{figure}[t]
    \centering
    \subfigure[]{\includegraphics[width = 0.48\textwidth, trim = 5 0 30 20, clip]{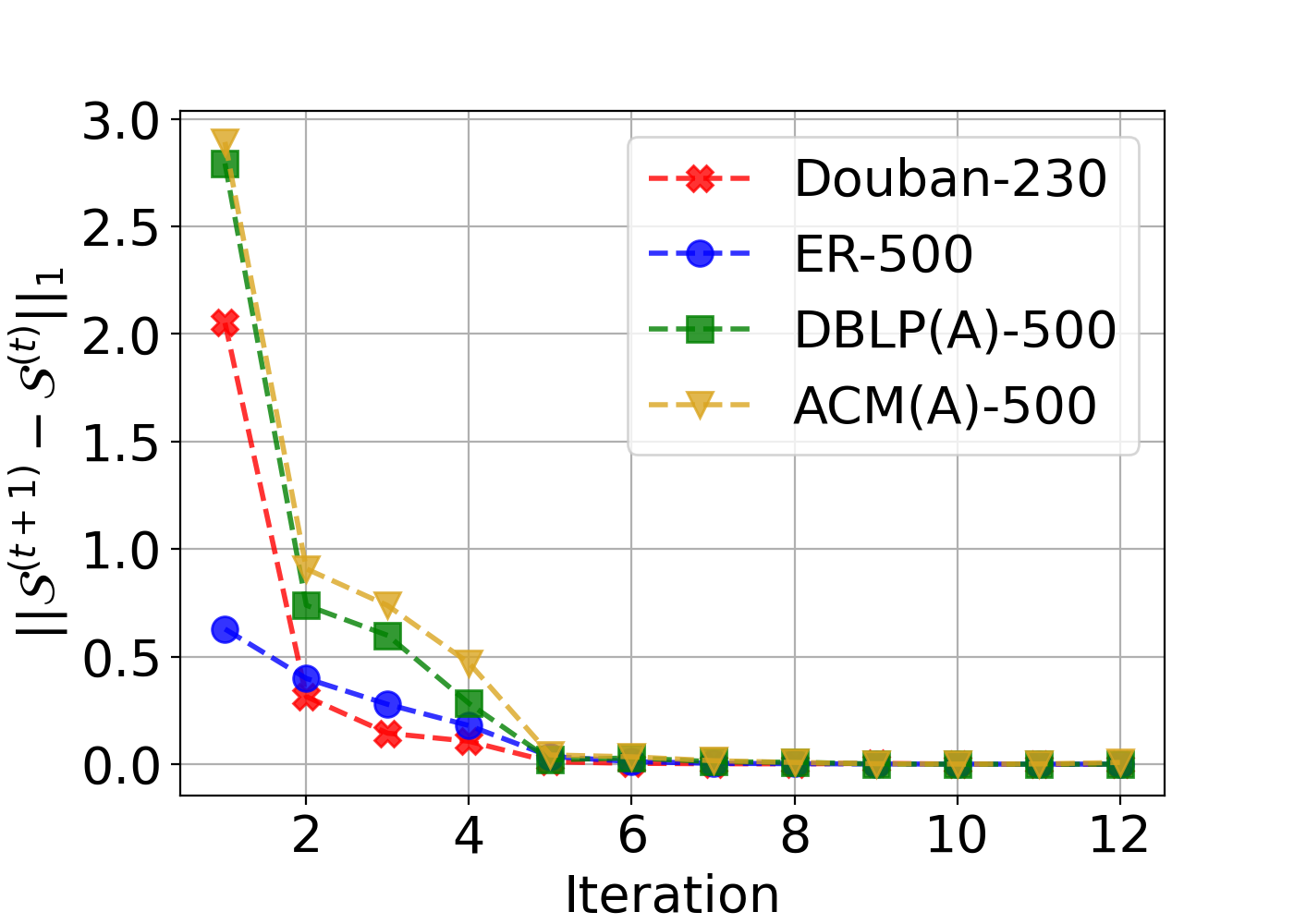}\label{fig:conv-1}}
    \subfigure[]{\includegraphics[width = 0.48\textwidth, trim = 5 0 30 20, clip]{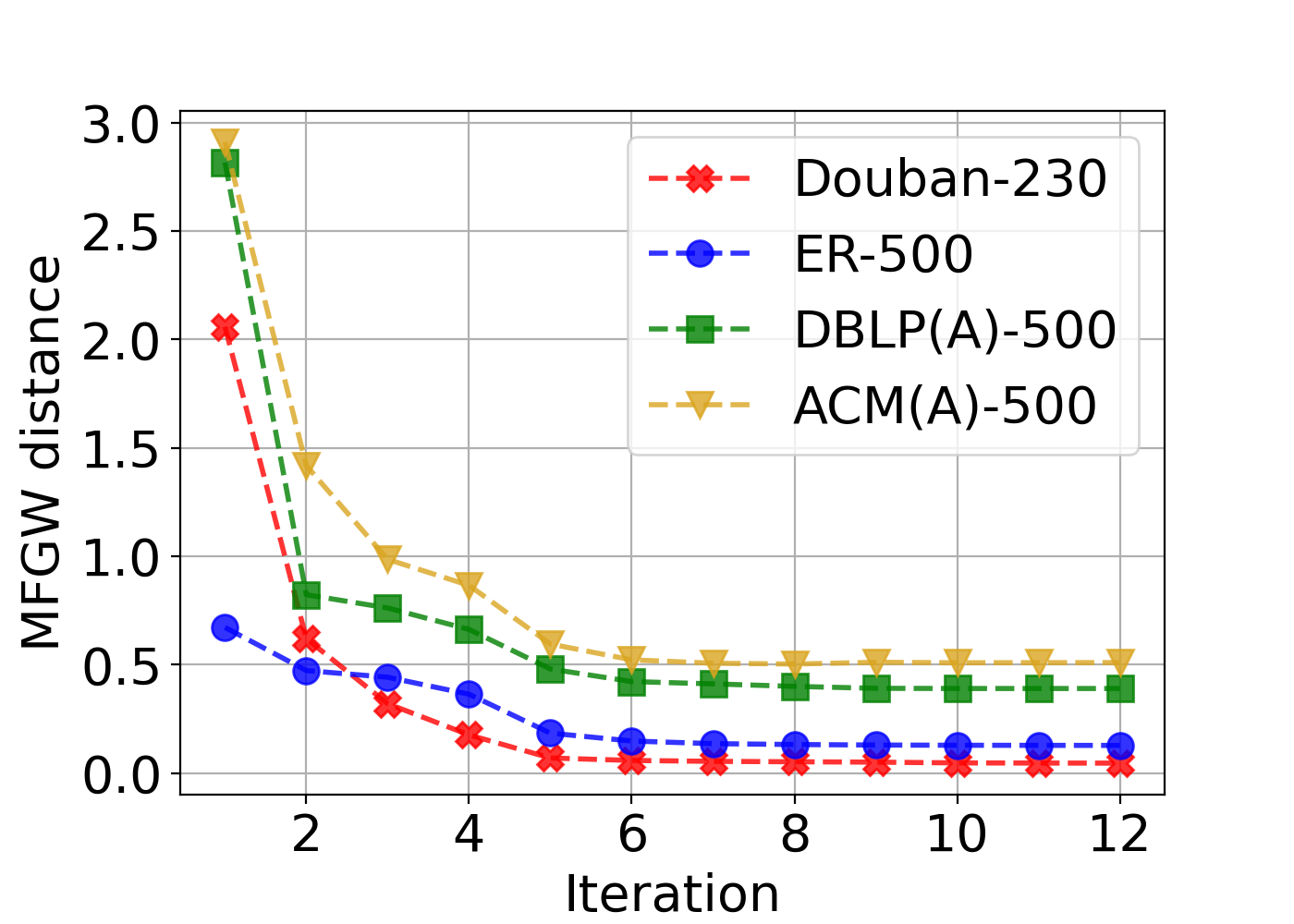}\label{fig:conv-2}}
    \caption{Convergence of the proximal point method: (a) Difference between two successive OT couplings (i.e., $\|\ts^{(t+1)}-\ts^{(t)}\|_1$); (b) MFGW distances along optimization.}
    \label{fig:conv}
\end{figure}
\subsection{Convergence Analysis}\label{sec:exp-conv}
We empirically validate the convergence guarantee in Proposition~\ref{prop:converge}. We evaluate the difference between two successive OT coupling tensors $\ts$, i.e., $\|\ts^{(t+1)}-\ts^{(t)}\|_1$, and the values of the MFGW distance along the proximal point optimization. As shown in Figure~\ref{fig:conv}, the objective function (i.e., MFGW distance) is non-increasing and the solution (i.e., $\ts$) converges along the proximal point iteration.

\subsection{Hyperparameter Study}\label{sec:exp-hyper}
We study the sensitivity to hyperparameters $\alpha$ and $\beta$ on the DBLP(A)-500 dataset, and results are shown in Figure~\ref{fig:hyper}. In general, the performance of \name\ is stable in the whole feasible hyperparameter space. Both metrics are quite robust to the restart probability $\beta$ but tend to decrease as $\alpha$ approaches 1. This is because attributes in DBLP(A)-500 dataset are quite informative, and the alignment almost exclusively depends on the structure when $\alpha\to 1$. Note that even under such an extreme case, \name\ still achieves good performance.
\begin{figure}[t]
    \centering
    \subfigure[]{\includegraphics[width = 0.47\linewidth, trim = 135 20 75 40, clip]{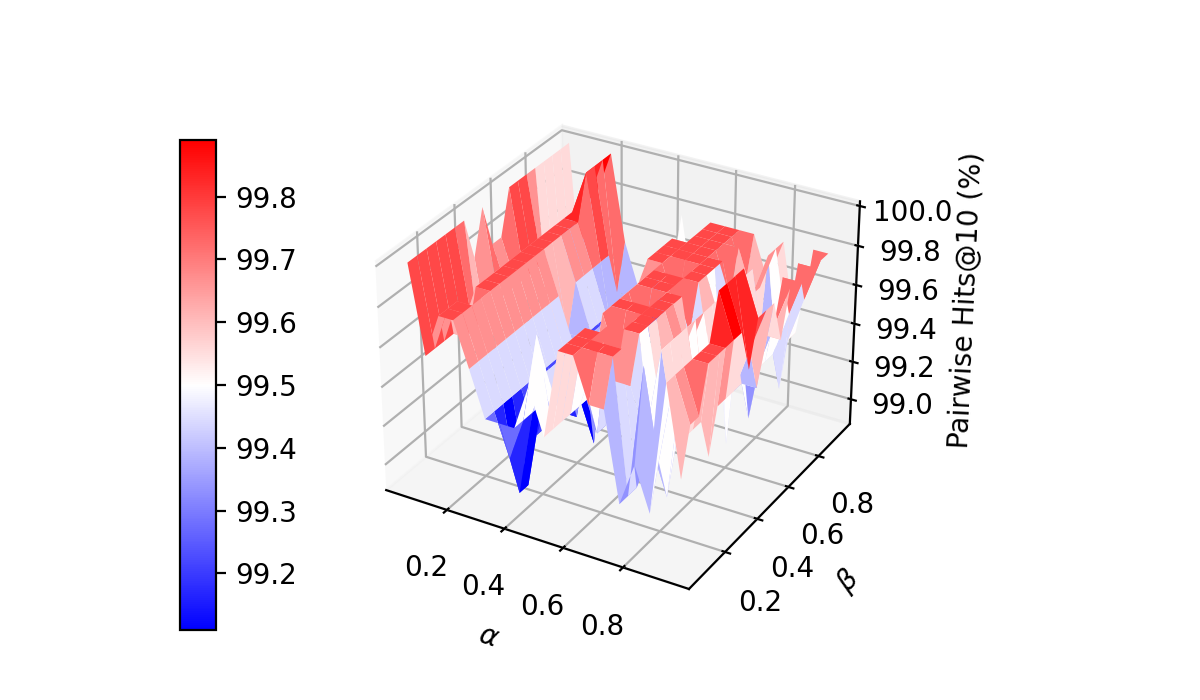}}\quad
    \subfigure[]{\includegraphics[width = 0.47\linewidth, trim = 135 20 75 40, clip]{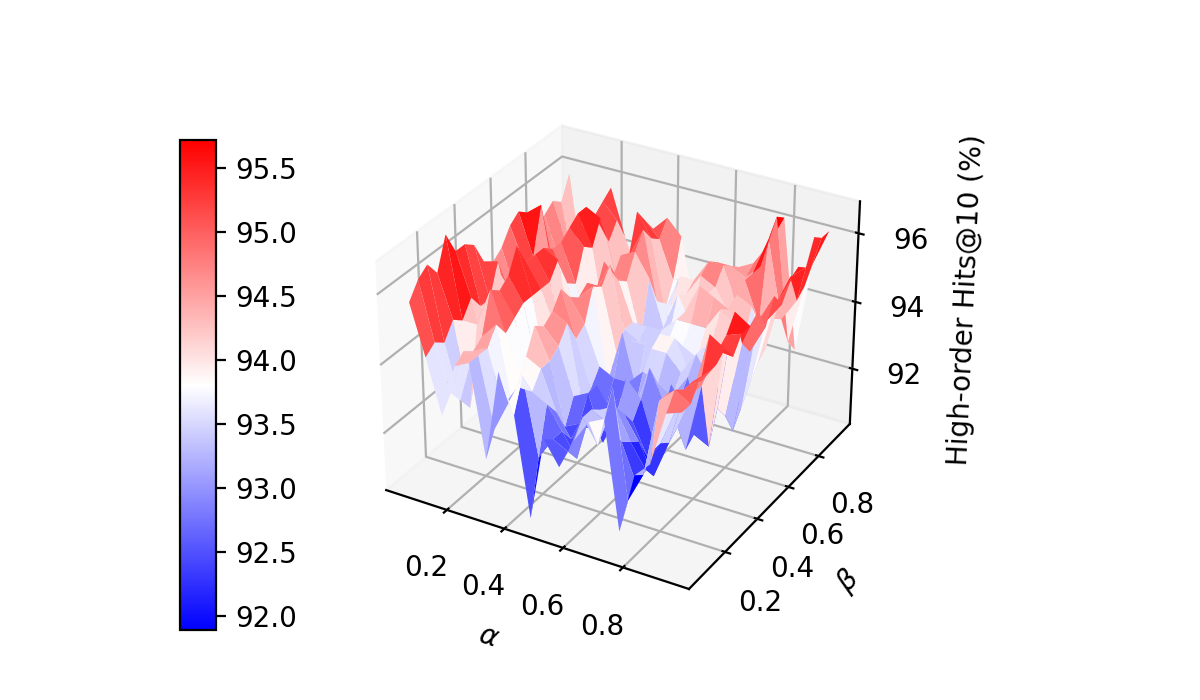}}
    \caption{Hyperparameter study on DBLP(A)-500: (a) Pairwise Hits@10; (b) High-order Hits@10.}
    \label{fig:hyper}
\end{figure}

%% file: 6-related.tex
\section{RELATED WORK}\label{sec:related}
\noindent\textbf{Network alignment.} Extensive efforts have been made for pairwise network alignment. Consistency-based methods~\cite{koutra2013big,singh2008global,zhang2016final} are built upon the alignment consistency principle, which assumes similar node pairs to have similar alignment results. Embedding-based methods~\cite{chu2019cross,du2019joint,zhang2021balancing,zhang2020nettrans,liu2016aligning,yan2021dynamic,yan2023reconciling,liu2022joint} generate low-dimensional node embeddings by pushing anchor node pairs as close as possible in the embedding space, enforced by contrastive loss~\cite{jing2021hdmi,jing2022x}. OT-based methods~\cite{maretic2022fgot,xu2019scalable,xu2019gromov,zeng2023parrot,zhao2020relaxed,zhao2020semi,wang2023wasserstein} represent graphs as distributions and optimize the distribution alignment by minimizing the Wasserstein-like discrepancy. To reduce the time complexity, a line of work leverages the hierarchy of graphs by finding clusters in graphs ~\cite{zhang2019multilevel,jing2022coin,wang2022unsupervised} to accelerate the pairwise alignment following a \textit{coarsen-align-interpolate} strategy. For multi-network alignment, the transitivity constraint was first proposed to ensure the consistency between alignments for different network pairs~\cite{chu2019cross,zhang2015multiple,zhou2020unsupervised}, but is hard to handle due to the non-convexity. The product graph, whose size becomes intractable in the multi-network setting, is adopted to model high-order alignment consistency~\cite{du2021sylvester,li2021scalable}. 
The low-rank approximation is explored for fast implicit solutions~\cite{du2021sylvester,li2021scalable,liu2016cross}, but high complexities are still inevitable when reconstructing the explicit solution.

\noindent\textbf{Optimal transport on graphs.} OT has achieved great success in handling geometric data such as graphs. Apart from network alignment reviewed above, OT has been applied in other graph-related tasks, including graph comparison~\cite{titouan2019optimal}, graph clustering~\cite{xu2019scalable}, graph compression~\cite{garg2019solving} and graph representation learning~\cite{vincent2021online,zeng2023generative}. The superiority of OT on graphs is rooted in its ability to capture the intrinsic geometry~\cite{peyre2016gromov} and compare objects in different metric spaces, e.g., graphs~\cite{memoli2011gromov}. While most OT literature is restricted to the two marginal cases, a few works~\cite{beier2022multi,fan2022complexity,pass2015multi} study the multi-marginal setting suffering from the complex problem formulation and intractable computation. In this paper, we fill this gap by generalizing the FGW distance to the multi-marginal setting, followed by a fast proximal point solution with convergence to a local optimum.

%% file: 7-con.tex
\section{CONCLUSION}\label{sec:con}
In this paper, we study the multi-network alignment problem from the view of hierarchical multi-marginal optimal transport (MOT). To handle the high computational complexity, we explore the hierarchical structure of graph data and decompose the original problem into smaller cluster-level and node-level alignment subproblems. To depict high-order node relationships, a position-aware cost is first generated based on the unified random walk with restart (RWR) embeddings, and the multi-marginal fused Gromov-Wasserstein (MFGW) distance is further leveraged to align multiple networks jointly. A fast algorithm is proposed with guaranteed convergence to a local optimum, reducing both time and space complexities by an exponential factor in terms of the number of networks compared with the straightforward solution. Extensive experiments validate the effectiveness and scalability of the proposed \name\ on the multi-network alignment task.

%% file: appendix.tex
\clearpage

\section{Detailed Algorithm}\label{app:algo}
\begin{algorithm}[htbp]
\small
    \caption{Cluster-level alignment}\label{algo:cluster}
    \KwInput{$K$ graphs $\mathcal{G}_i(\mathbf{A}_i,\mathbf{X}_i)$, hyperparameter $\alpha,\lambda,T,L$.}
    \KwOutput{The cluster-level alignments $\mathcal{C}_i^j$.}
    Initialize $\mathbf{X}_b^{(0)}\!\!,\mathbf{A}_b^{(0)}\!\!,\bm{\mu}_b,\bm{\mu}_i,\mathbf{S}_i^{(0)}\!=\!\bm{\mu}_i\!\!\otimes\!\bm{\mu}_b,\forall i\!=\!1,2,...,K$\;
    \For{$t=0,1,\dots,T$}{
        \For{$i=1,2,\dots,K$}{
            Compute cross-cost $\mathbf{C}_{\text{cross}_i}^{(t)}(v,u)=\|\mathbf{X}_i(v)-\mathbf{X}_b^{(t)}(u)\|_2$\;
            Compute modified cost $\mathbf{Q}^{(t)}_i\!\!=\!(1\!-\!\alpha)\mathbf{C}^{(t)}_{\text{cross}_i} \!+\! \alpha\mathbf{L}_i^{(t)}\!\!-\!\lambda\log\mathbf{S}_i^{(t)}\!$\;
            \For{$l=0,1,\dots,L$}{
                Compute scaling vectors $\mathbf{u}_i^{(l)}\!\!,\mathbf{u}_b^{(l)}$ by Eq.~\eqref{eq:sinkhorn1};
            }
            Update $\mathbf{S}_i^{(t+1)}$ by Eq.~\eqref{eq:sinkhorn2}\;
        }
        Update $\mathbf{A}_b^{(t+1)}\!\!=\!\frac{\mathbf{1}_{M\times M}}{\bm{\mu}_b\bm{\mu}_b^{\T}}\sum_{i=1}^{K}\!\!\left(\mathbf{S}_i^{(t+1)^{\T}}\!\mathbf{A}_i\mathbf{S}_i^{(t+1)}\right)$\;
        Update $\mathbf{X}^{(t+1)}_b\!\!=\!\sum_{i=1}^{K}\!\!\left(\!\text{diag}\left(\frac{\mathbf{1}_{M}}{\bm{\mu}_b}\right)\!\mathbf{S}_i^{(t+1)^{\T}}\!\mathbf{X}_i\!\right)$\;
    }
    Form cluster alignments $\mathcal{C}_i^j=\{v_i\}$ for $v_i\in\G_i$ such that $v_j=\arg\max_{v\in\G_b}\mathbf{S}_i^{(T+1)}(v_i,v)$\;
    \Return Cluster alignments $\C^j=\bigcup_{i=1}^K\C_i^j,\forall j=1,\dots,M$.
\end{algorithm}

\begin{algorithm}[htbp]
\small
    \caption{Node-level alignment}\label{algo:node}
    \KwInput{$K$ graphs $\mathcal{G}_i(\!\mathbf{A}_i,\!\mathbf{X}_i)$, cross-cost tensor $\tc$, hyperparameter $\alpha,\lambda,T,L$.}
    \KwOutput{The node-level alignment tensor $\ts$.}
    Initialize $\bm{\mu}_i,\forall i=1,2,\dots,K$ and a feasible solution $\ts^{(0)}=\bigotimes_{i=1}^K\bm{\mu}_i$\;
    \For{$t=0,1,\dots,T$}{
        Compute $\bm{\mathcal{Q}}^{(t)}=(1-\alpha)\tc + \alpha\tl^{(t)}-\lambda\log\ts^{(t)}$\;
        \For{$l=0,1,\dots,L$}{
            Compute scaling vectors $\mathbf{u}_i^{(l)},\forall i = 1,2,\dots,K$ by Eq.~\eqref{eq:sinkhorn1}; 
        }
        Update $\ts^{(t+1)}$ by Eq.~\eqref{eq:sinkhorn2}\;
    }
    \Return Node alignment tensor $\ts^{(T+1)}$.
\end{algorithm}

\begin{algorithm}[htbp]
\small
    \caption{Hierarchical multi-marginal optimal transport (\name)}\label{algo:hot}
    \KwInput{$K$ graphs $\mathcal{G}_i(\!\mathbf{A}_i,\!\mathbf{X}_i)$, anchor node sets $\mathcal{L}$, cluster number $M$, hyperparameter $\alpha,\lambda,L,T$.}
    \KwOutput{The alignment tensor $\ts$.}
    Compute RWR scores $\mathbf{R}_i,\forall i\!=\!1,2,\dots,K$ by Eq.~\eqref{eq:rwr}\;
    Compute cluster-level alignment $\C^j,\forall j\!=\!1,2,\dots,M$ by Algo.~\ref{algo:cluster}\;
    \For{$j=1,2,\ldots,M$}{
        Compute cross-cost tensor $\tc^j$ for cluster $\mathcal{C}^j$ by Eq.~\eqref{eq:costtensor}\;
        Compute node alignment tensor $\ts^j$ by Algorithm~\ref{algo:node}\;
    }
    \Return $\ts=\text{diag}(\ts^1,\ldots,\ts^M)$.
\end{algorithm}

\section{Proof}\label{app:proof}
\mfgw*
\begin{proof}
    The Wasserstein term, i.e., $\langle(1-\alpha)\tc,\ts\rangle$, is straightforward, and we mainly focus on re-formulating the Gromov-Wasserstein part into the tensor form. The Gromov-Wasserstein term in Eq.~\eqref{eq:mfgw1} can be written as:
    \begin{equation*}
            \small
        \begin{aligned}
        &\sum_{\substack{1\leq j<k\leq K\\v_{1},...,v_{K}\\ v_1',..., v_K'}}\left|\mathbf{C}_j(v_{j},\!v_{j'})-\mathbf{C}_k(v_{k},\!v_{k}')\right|^2\ts(v_{1},...,v_{K})\ts(v_{1}',...,v_{K}')\\
        =&\!\!\!\sum_{v_{1},...,v_{K}}\!\!\!\underbrace{\left[\!\sum_{1\leq j<k\leq K\atop v_{1}',...,v_{K}'} \!\!\!\!\left|\mathbf{C}_j(v_{j},\!v_{j}')\!-\!\mathbf{C}_k(v_{k},\!v_{k}')\right|^2\!\!\ts(v_{1}',...,v_{K}')\!\right]}_{\tl(v_1,...,v_K)}\!\!\ts(v_{1},...,v_{K})\\
        =&\langle\tl,\ts\rangle
    \end{aligned}
    \end{equation*}
    We further simplify $\tl(\mathbf{C}_1,...,\mathbf{C}_K,\tc,\ts)$ as follows:
    \begin{equation*}
    \small
        \begin{aligned}
        &\tl(v_1,...,v_K)\\
        =&(K-1)\sum_{j,v_{j}'} \mathbf{C}_j(v_{j},v_{j}')^2\underbrace{\sum_{v_{1}',...,v_{j-1}',v_{j+1}',... v_{K}'}\ts(v_{1}',... v_{K}')}_{\mathcal{P}_j(\ts)(v_{j}')}\\
        -&2\sum_{1\leq j<k\leq K}\sum_{v_{j}'}\mathbf{C}_j(v_{j},v_{j}')\sum_{v_{k}'}\mathbf{C}_k(v_{k},v_{k}')\underbrace{\sum_{\text{w/o }v_{j}', v_{k}'}\ts(v_{1}',... v_{K}')}_{\mathcal{P}_{i,j}(\ts)(v_{j}',v_{k}')}\\
        =& (K\!-\!1)\sum_{j=1}^K\mathbf{C}_j^2(v_j,\cdot)\mathcal{P}_j(\ts) - 2\!\!\!\!\!\!\sum_{1\leq j<k\leq K}\!\!\!\!\!\!\mathbf{C}_j(v_j,\cdot)\mathcal{P}_{j,k}(\ts)\mathbf{C}_k(v_k,\cdot)^\T
    \end{aligned}
    \end{equation*}
\end{proof}

\complexity*
\begin{proof}
    For space complexity, thanks to the block diagonal property of $\ts$, only non-zero elements in the small blocks need to be stored. Storing $K$ attributed networks takes $\mathcal{O}(K(m+nd))$ space. Storing cost tensors $\tc^j$ and alignment tensors $\ts^j$ each takes $\mathcal{O}(\overline{n}^K)$ space. Therefore, the overall space complexity for \name\ is $\mathcal{O}(M\overline{n}^K)$, which is an exponential reduction of space in terms of the number of graphs $K$ compared with the $\mathcal{O}(n^K)$ space complexity given by the straightforward method.

    For time complexity, the complexity is $\mathcal{O}(mn)$\footnote{The cost can be reduced to be linear in $m$ with the advancement of faster Sylvester Equation solver~\cite{du2018fasten}.} for the RWR calculation in Eq.~\eqref{eq:rwr}and $\mathcal{O}(M\overline{n}^K(|\mathcal{L}|+d))$ for the cost tensor calculation in Eq.~\eqref{eq:costtensor}. For cluster-level alignment calculation, the complexities are $\mathcal{O}(KTLMn)$ for the optimal couplings $\mathbf{S}_i$ in Eq.~\eqref{eq:opt_s}, $\mathcal{O}(TKMn^2)$ for adjacency matrices $\mathbf{A}_b$ in Eq.~\eqref{eq:opt_a} and $\mathcal{O}(TKMn(|\mathcal{L}|+d))$ for node attribute matrices $\mathbf{X}_b$ in Eq.~\eqref{eq:opt_x}~\cite{titouan2019optimal}. Thus, the complexity of cluster-level alignment is $\mathcal{O}(TKMn^2)$. For node-level alignment, the complexity lies in computing $\bm{\mathcal{Q}}$, $\mathbf{u}_i$ and $\ts$ (Eqs.~\eqref{eq:proximal}-\eqref{eq:sinkhorn2}) iteratively and equals $\mathcal{O}(TMK^2\overline{n}^K)$. Therefore, the overall time complexity of \name\ is $\mathcal{O}(TKM(n^2+K\overline{n}^K))$.
\end{proof}

\converge*
\begin{proof}
    We denote the objective function of MFGW distance under $q=2$ as:
    \begin{align*}
        &f(\ts) = \min_{\mathbf{S}\in\Pi(\bm{\mu}_1,\dots,\bm{\mu}_K)}\langle(1-\alpha)\tc + \alpha\tl,\ts\rangle\\
    \end{align*}
    where $\tl(v_1,...,v_K) = (K-1)\sum_{1\leq j\leq K}\mathbf{C}_j^2(v_j,\cdot) \mathcal{P}_j(\ts) - 2\sum_{1\leq j<k\leq K}\mathbf{C}_j(v_j,\cdot)\mathcal{P}_{j,k}(\ts)\mathbf{C}_k(v_k,\cdot)^\T$. The proximal point method decompose the above non-convex problem into a sequence of entropy-regularized multi-margianl OT problem as follows:
    \begin{equation*}
    \small
        \begin{aligned}
            u(\ts,\ts^{(t)})
            =&\!\min_{\mathbf{S}\in\Pi(\bm{\mu}_1,...,\bm{\mu}_K)}\!\langle(1-\alpha)\tc \!+\! \alpha\tl^{(t)},\ts\rangle \!+\! \text{KL}(\ts\|\ts^{(t)})
        \end{aligned}
    \end{equation*}
    where $\tl^{(t)}$ denotes $\tl$ with $\ts=\ts^{(t)}$. Note that the Gromov-Wasserstein term $\tl^{(t)}$ is a constant tensor in the above equation as $\ts^{(t)}$ is fixed during the optimization. Besides, it is obvious that the solution space $\mathcal{X}=\Pi(\bm{\mu}_1,\dots,\bm{\mu}_K)$ is a closed convex set. We then evaluate the following conditions to hold the global convergence of the proximal point method based on Theorem 1 in~\cite{razaviyayn2013unified}:
    \begin{itemize}
        \item \verb|C1|: $u(\ts,\ts)=f(\ts),\forall \ts\in\mathcal{X}$.
        \item \verb|C2|: $u(\ts,\ts')\geq f(\ts), \forall \ts,\ts'\in\mathcal{X}$.
        \item \verb|C3|: $u(\ts,\ts';d)|_{\ts=\ts'}=f'(\ts,d), \forall d$ with $\ts'+d\in\mathcal{X}$.
        \item \verb|C4|: $u(\ts,\ts')$ is continuous in $(\ts,\ts')$.
    \end{itemize}
    
    For \verb|C1|, it is obvious that $u(\ts,\ts) = f(\ts),\forall \ts\in\mathcal{X}$.

    For \verb|C2|, since $\text{KL}(\ts\|\ts^{(t)})\geq 0$ and the equation holds when $\ts=\ts^{(t)}$, we have $u(\ts,\ts')\geq f(\ts),\forall \ts,\ts'\in\mathcal{X}$.

    For \verb|C3|, according to Proposition 1 in~\cite{razaviyayn2013unified}, \verb|C3| holds when \verb|C1| and \verb|C2| hold.

    For \verb|C4|, it is easy to validate that $u(\ts,\ts')$ is continuous.

    Therefore, the proximal point method has global convergence to a stationary point of the MFGW problem.
\end{proof}

\bound*
\begin{proof}
    For clarity, we slightly abuse the notation $\text{FGW}_{2,\alpha}(\G_j,\G_k;\mathbf{S})$ to denote the value of FGW distance under coupling $\mathbf{S}$. We first focus on the Wasserstein distance part in the MFGW distance in Eq.~\eqref{eq:mfgw1}, which can be reformulated as follows:
    \begin{equation}\label{eq:mwd}
    \small
        \begin{aligned}
            &\sum_{v_{1},\dots,v_{K}}\sum_{1\leq j<k\leq K}\|\mathbf{Z}_j(v_{j})-\mathbf{Z}_k(v_{k})\|_2\ts(v_{1},\dots,v_{K})\\
            =&\sum_{1\leq j<k\leq K}\sum_{v_{j},v_{k}}\|\mathbf{Z}_j(v_{j})-\mathbf{Z}_k(v_{k})\|_2\sum_{\text{w/o }v_{j},v_{k}}\ts(v_{1},\dots,v_{K})\\
            =&\sum_{1\leq j<k\leq K}\sum_{v_{j},v_{k}}\|\mathbf{Z}_j(v_{j})-\mathbf{Z}_k(v_{k})\|_2\mathcal{P}_{j,k}(\ts)(v_{j},v_{k})
        \end{aligned}
    \end{equation}
    We then consider the Gromov-Wasserstein distance pair in the MFGW distance in Eq.~\eqref{eq:mfgw1}, which can be reformulated as follows:
    \begin{equation}\label{eq:mgwd}
    \small
        \begin{aligned}
            &\!\!\!\sum_{\substack{1\leq j<k\leq K\\v_{j},v_{j}'\\ v_{k}, v_{k}'}}\!|\mathbf{C}_j(v_{j},v_{j}')-\mathbf{C}_k(v_{k},v_{k}')|^2\!\ts(v_{1},...,v_{K})\ts(v_{1}',...,v_{K}')\\
            =&\!\!\!\sum_{\substack{1\leq j<k\leq K\\v_{j},v_{j}'\\ v_{k}, v_{k}'}}\!\!\!|\mathbf{C}_j(v_{j},v_{j}')\!-\!\mathbf{C}_k(v_{k},v_{k}')|^2\!\!\!\sum_{\text{w/o }v_{j},v_{k}}\!\!\!\ts(v_{1},...,v_{K})\!\!\!\sum_{\text{w/o }v_{j}',v_{k}'}\!\!\!\ts(v_{1}',...,v_{K}')\\
            =&\!\!\!\sum_{\substack{1\leq j<k\leq K\\v_{j},v_{j}'\\ v_{k}, v_{k}'}}\!|\mathbf{C}_j(v_{j},v_{j}')-\mathbf{C}_k(v_{k},v_{k}')|^2\mathcal{P}_{j,k}(\ts)(v_{j},v_{k})\mathcal{P}_{j,k}(\ts)(v_{j}',v_{k}')
        \end{aligned}
    \end{equation}
    Combining Eqs.~\eqref{eq:mwd} and~\eqref{eq:mgwd}, the MFGW distance can be reformulated as follows:
    \begin{equation*}
    \small
        \begin{aligned}
            &\text{MFGW}_{2,\alpha}(\G_1,\G_2,\dots,\G_K)\\
            =&\!\sum_{1\leq j<k\leq K}\min_{\ts\in\Pi(\bm{\mu}_1,\dots,\bm{\mu}_K)}\!\!\left[\!(1-\alpha)\!\sum_{v_{j},v_{k}}\!\|\mathbf{Z}_j(v_{j})\!-\!\mathbf{Z}_k(v_{k})\|_2\mathcal{P}_{j,k}(\ts)(v_{j},v_{k})\right.\\
            +&\left.\alpha\!\sum_{v_{j},v_{k}\atop v_{j}',v_{k}'}|\mathbf{C}_j(v_{j},v_{j}')-\mathbf{C}_k(v_{k},v_{k}')|^2\mathcal{P}_{j,k}(\ts)(v_{j},v_{k})\mathcal{P}_{j,k}(\ts)(v_{j}',v_{k}')\right]\\
            =&\!\sum_{1\leq j<k\leq K}\text{FGW}_{2,\alpha}(\G_j,\G_k;\mathcal{P}_{j,k}(\ts))
        \end{aligned}
    \end{equation*}
    Note that $\mathcal{P}_{j,k}(\ts)\in\Pi(\bm{\mu}_j,\bm{\mu}_k)$ is a suboptimal coupling of $\text{FGW}_{2,\alpha}(\G_j,\G_k)$, that is:
    \begin{equation*}
    \small
        \text{FGW}_{2,\alpha}(\G_j,\G_k;\mathcal{P}_{j,k}(\ts))\geq \text{FGW}_{2,\alpha}(\G_j,\G_k)
    \end{equation*}
    Therefore, we can prove the proposed upper bound as follows:
    \begin{equation*}
    \small
        \begin{aligned}
            \text{MFGW}_{2,\alpha}(\G_1,\G_2,\dots,\G_K)&=\!\!\sum_{1\leq j<k\leq K}\!\!\text{FGW}_{2,\alpha}(\G_j,\G_k;\mathcal{P}_{j,k}(\ts))\\
            &\geq \!\!\sum_{1\leq j<k\leq K}\!\!\text{FGW}_{2,\alpha}(\G_j,\G_k)
        \end{aligned}
    \end{equation*}
\end{proof}

\vspace{-.5\baselineskip}
\section{Additional Experiments}\label{app:exp}
We carried out additional experiments to evaluate our proposed \name\ from different aspects.
\begin{figure*}[htbp]
    \centering
    \includegraphics[width = \textwidth,trim = 0 0 0 0,clip]{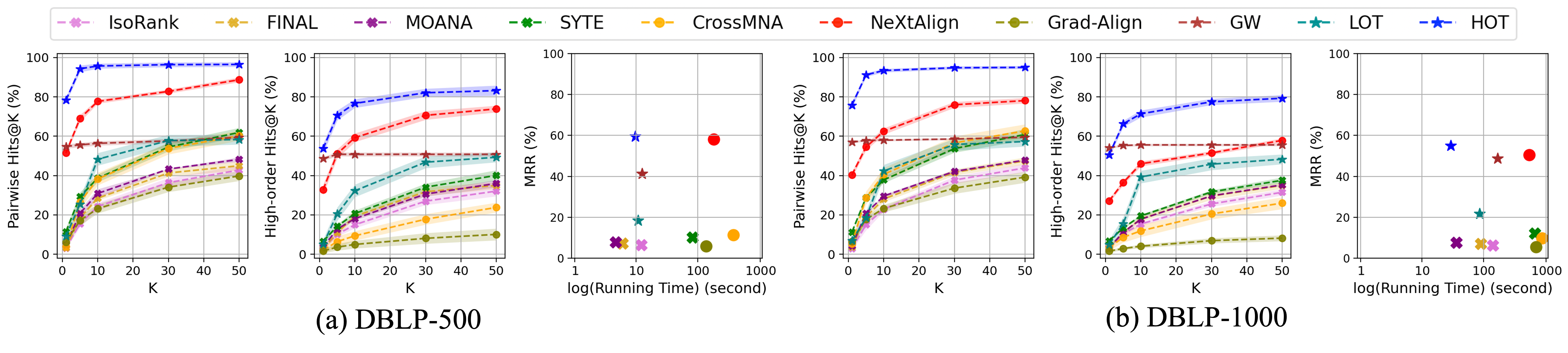}
    \vspace{-20pt}
    \caption{Alignment results on plain networks: (a) DBLP-500; (b) DBLP-1000.}\label{fig:plain-2}
\end{figure*}

\begin{figure*}[htbp]
    \centering
    \includegraphics[width = \textwidth,trim = 0 0 0 0,clip]{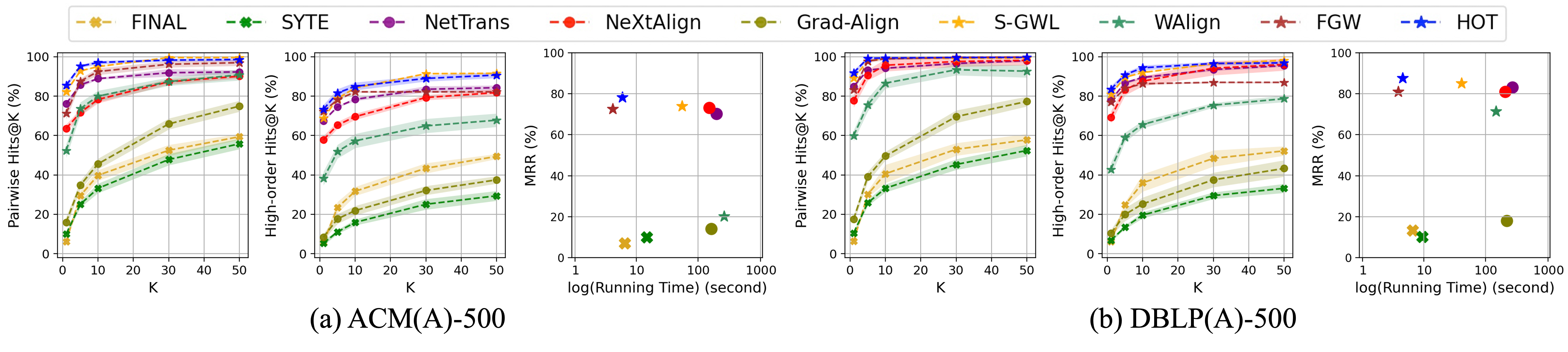}
    \vspace{-20pt}
    \caption{Alignment results on attributed networks: (a) ACM(A)-500; (b) DBLP(A)-500.}\label{fig:attributed-2}
\end{figure*}

\paragraph{More effectiveness results.}
We provide more experimental results on plain networks: DBLP-500/DBLP-1000, and attributed networks: DBLP(A)-500/ACM(A)-500. Results are shown in Figures~\ref{fig:plain-2} and \ref{fig:attributed-2}, respectively. Our proposed \name\ consistently outperforms all baselines on these datasets in both pairwise and high-order Hits@10 metrics. Comparing the performance of \name\ on DBLP-500 and DBLP(A)-500, \name\ achieves nearly 100\% pairwise Hits@{\em K} on DBLP-500 even without the node attributes, which is comparable to the performance on DBLP(A)-500. This validates that the position-aware cost tensor design extracts the essential positional information and is quite beneficial to plain network alignment.

We also test the proposed \name\ on larger graphs, including the plain DBLP and attributed DBLP networks with up to 6000 nodes. We set the cluster number as $M=\lceil\frac{n}{50}\rceil$. Experiment results are shown in Tables~\ref{tab:large-1} and~\ref{tab:large-2}. It is shown that the proposed \name\ can achieves excellent performance when aligning large multiple networks. Besides, as stated before, all baseline methods constructing a dense alignment tensor fail to handle all the large multi-network alignment problems (with more than 2000 nodes) due to OOM. 
\begin{table*}[]
\small
\begin{tabular}{@{}l|ccccc|ccccc|c@{}}
\toprule
\multirow{2}{*}{\#nodes} & \multicolumn{5}{c|}{Pairwise Hits}                   & \multicolumn{5}{c|}{High-order Hits}                 & \multirow{2}{*}{MRR} \\ \cmidrule(lr){2-11}
                         & @1       & @5       & @10      & @30      & @50      & @1       & @5       & @10      & @30      & @50      &                      \\ \midrule
2000                     & 78.0±1.3 & 91.2±1.0 & 92.9±1.0 & 94.2±0.9 & 94.5±0.9 & 55.1±2.1 & 67.8±1.6 & 71.7±1.5 & 76.8±1.7 & 78.2±1.7 & 60.7±1.8             \\
3000                     & 69.5±1.9 & 82.8±1.8 & 84.8±1.7 & 86.5±1.4 & 86.9±1.3 & 40.9±1.4 & 52.6±1.2 & 56.7±1.3 & 61.3±1.4 & 62.5±1.4 & 46.3±1.3             \\
4000                     & 65.1±0.8 & 80.7±0.9 & 83.2±0.8 & 85.0±0.8 & 85.5±0.8 & 38.0±0.6 & 49.6±0.8 & 54.0±1.0 & 59.6±0.7 & 61.4±0.8 & 43.4±0.7             \\
5000                     & 76.9±0.9 & 91.5±0.6 & 93.5±0.5 & 94.5±0.5 & 94.6±0.5 & 53.9±0.9 & 67.2±0.9 & 72.2±1.0 & 77.6±1.0 & 79.1±1.0 & 59.8±0.9             \\
6000                     & 64.7±0.9 & 81.1±0.8 & 83.7±0.7 & 85.6±0.6 & 86.1±0.6 & 37.5±0.8 & 49.3±0.8 & 54.2±0.8 & 60.3±1.0 & 62.3±0.9 & 42.9±0.8             \\ \bottomrule
\end{tabular}
\caption{Experiments on aligning larger DBLP plain networks.}
\label{tab:large-1}
\end{table*}

\begin{table*}[]
\small
\begin{tabular}{@{}l|ccccc|ccccc|c@{}}
\toprule
\multirow{2}{*}{\#nodes} & \multicolumn{5}{c|}{Pairwise Hits}                   & \multicolumn{5}{c|}{High-order Hits}                 & \multirow{2}{*}{MRR} \\ \cmidrule(lr){2-11}
                         & @1       & @5       & @10      & @30      & @50      & @1       & @5       & @10      & @30      & @50      &                      \\ \midrule
2000                     & 91.7±0.6 & 98.1±0.3 & 98.6±0.3 & 98.9±0.3 & 99.0±0.3 & 82.2±0.8 & 88.8±0.7 & 91.6±0.6 & 94.4±0.6 & 94.9±0.6 & 85.1±0.6             \\
3000                     & 91.0±0.4 & 97.0±0.4 & 97.4±0.4 & 97.6±0.4 & 97.7±0.4 & 78.6±0.7 & 85.5±0.8 & 88.0±0.8 & 89.6±0.6 & 89.9±0.6 & 81.6±0.7             \\
4000                     & 88.1±0.4 & 96.3±0.3 & 97.1±0.3 & 97.4±0.3 & 97.4±0.3 & 74.4±0.8 & 82.4±0.7 & 85.3±0.6 & 88.3±0.7 & 89.1±0.7 & 77.9±0.8             \\
5000                     & 88.9±0.3 & 97.8±0.2 & 98.7±0.2 & 99.0±0.2 & 99.1±0.2 & 78.2±0.4 & 86.5±0.4 & 90.1±0.4 & 93.8±0.5 & 94.6±0.5 & 81.9±0.4             \\
6000                     & 86.9±0.5 & 96.3±0.4 & 97.2±0.4 & 97.6±0.4 & 97.6±0.3 & 73.0±0.7 & 81.7±0.7 & 85.0±0.6 & 88.7±0.6 & 89.7±0.6 & 76.8±0.7             \\ \bottomrule
\end{tabular}
\caption{Experiments on aligning larger DBLP attributed networks.}
\label{tab:large-2}
\end{table*}

\paragraph{Effectiveness vs scalability.} 
We conducted an analysis to determine how the number of clusters ($M$) impacts the effectiveness and scalability of \name, and the results are shown in Figure~\ref{fig:perf}. In general, as the number of clusters $M$ increases, the alignment performance drops slightly, but the running time reduces dramatically. However, the alignment performance is still relatively stable when $M$ changes, with an average standard deviation of 2.6\%. The effect of $M$ on the alignment accuracy is two-fold. On one hand, increasing $M$ may result in aligned nodes being assigned to different clusters, which can be detrimental to the overall performance and may explain the decreasing trend of high-order Hits@10. On the other hand, if the cluster-level alignment is accurate, a higher value of $M$ can lead to a significant reduction in the search space, which may explain the intermediate performance increase.

Furthermore, comparing the running time with different $M$, the running time exhibits a huge drop when $M$ increases from 5 to 20, and stays relatively stable when $M$ further increases from 20 to 40. These results are consistent with our time complexity analysis in Section~\ref{sec:ana} and scalability results in Figure~\ref{fig:time}. Specifically, when $M$ is relatively small, the node-level alignment dominates the overall time complexity, resulting in an exponential decrease in running time. Conversely, when $M$ is relatively large, the cluster-level alignment dominates the overall time complexity, leading to a slow increase in running time.

\begin{figure}[!htbp]
    \centering
    \subfigure[ACM(A)-500]{\includegraphics[width = 0.48\linewidth, trim = 20 5 20 12, clip]{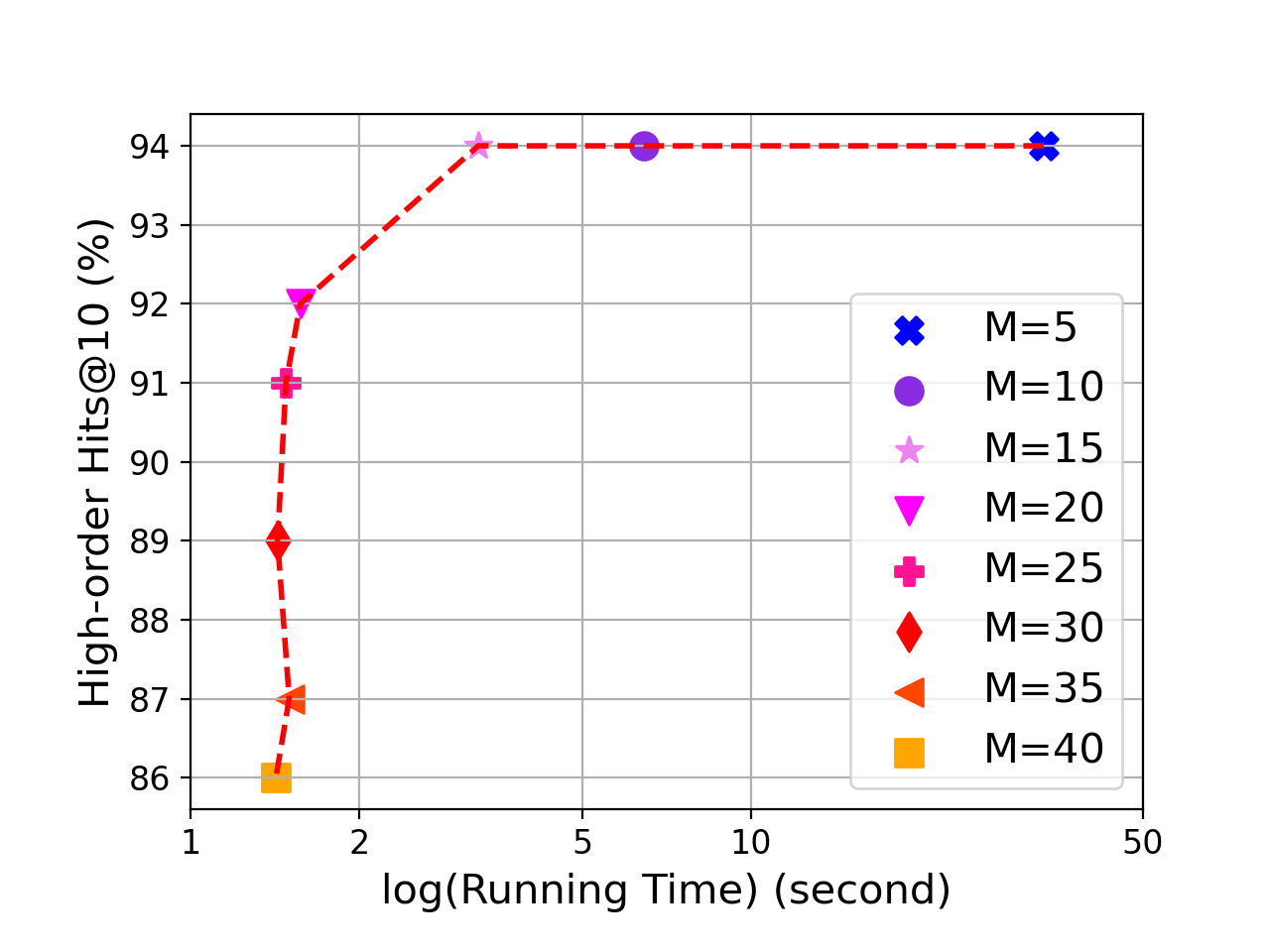}\label{fig:pair-M}}
    \subfigure[DBLP(A)-500]{\includegraphics[width = 0.48\linewidth, trim = 20 5 20 12, clip]{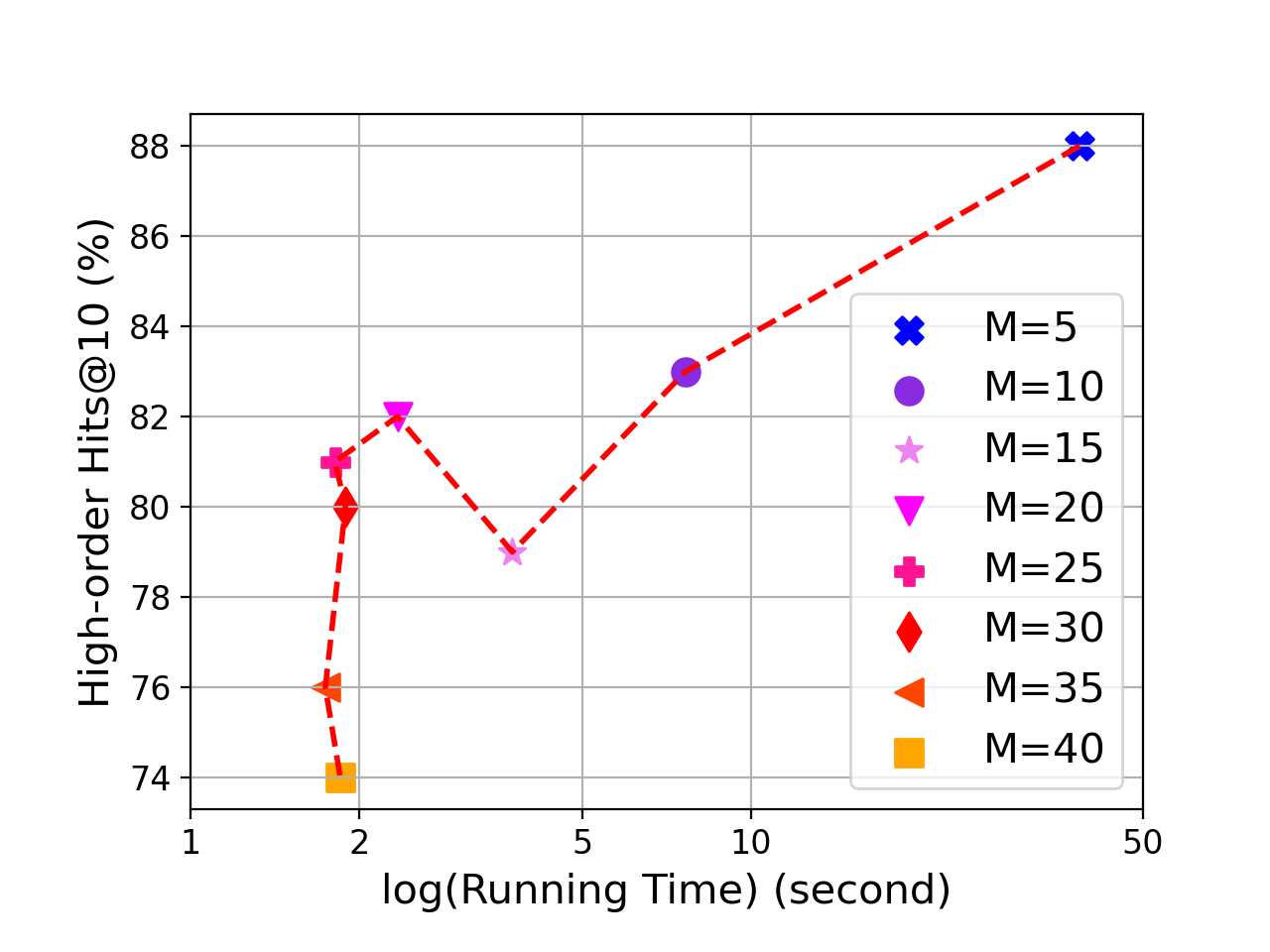}\label{fig:high-M}}
    \vspace{-10pt}
    \caption{High-order Hits@10 vs running time: (a) ACM(A)-500, (b) DBLP(A)-500. Number of clusters $M$ achieves a trade-off between effectiveness and scalability.}
    \label{fig:perf}
\end{figure}

\section{Reproducibility}\label{app:rep}
\begin{table*}[!htbp]
\centering
\begin{tabular}{@{}llllc@{}}
\toprule
Methods &Types &Scenarios &Settings &Hierarchical\\ \midrule
IsoRank~\cite{singh2008global} &consistency &plain &pairwise &\xmark\\
FINAL~\cite{zhang2016final} &consistency &plain\&attributed &pairwise &\xmark\\
MOANA~\cite{zhang2019multilevel} &consistency &attributed &pairwise &\cmark\\
SYTE~\cite{du2021sylvester} &consistency &plain\&attributed &multiple &\xmark\\ \midrule
CrossMNA~\cite{chu2019cross} &embedding &plain &multiple &\xmark\\
NetTrans~\cite{zhang2020nettrans} &embedding &attributed &pairwise &\xmark\\
NeXtAlign~\cite{zhang2021balancing} &embedding &plain\&attributed &pairwise &\xmark\\ 
Grad-Align~\cite{park2022grad} &embedding &attributed &pairwise &\xmark\\ \midrule
GW~\cite{memoli2011gromov}  &OT &plain &pairwise &\xmark\\
FGW~\cite{titouan2019optimal} & OT &attributed &pairwise &\xmark\\
S-GWL~\cite{xu2019scalable} &OT &attributed &pairwise &\cmark\\
LOT~\cite{scetbon2022linear} &OT &plain &pairwise &\xmark\\
WAlign~\cite{maretic2020wasserstein} &OT &attributed &pairwise &\xmark\\ \bottomrule
\end{tabular}
\caption{Baseline methods summary.}
\label{tab:base}
\end{table*}

\begin{table*}[!htbp]
    \setlength{\abovecaptionskip}{1mm}
    \setlength{\belowcaptionskip}{-5pt}
    \centering
    \begin{tabular}{@{}lllll@{}}
    \toprule
    Scenarios &Networks &\# nodes &\# edges &\# attributes\\ \midrule
    \multirow{4}{*}{Plain}
    &ER &500/500/500 &1813/2160/1982 &0\\
    &Douban &230/230/242 &336/344/356 &0 \\
    &DBLP-500 &500/500/500 &1710/1788/1785 &0 \\
    &DBLP-1000 &1000/1000/1000 &3627/3807/3807 &0 \\ \midrule
    \multirow{4}{*}{Attributed}
    &ACM(A)-500 &500/ 500/ 500 &1813/2160/1982 &17 \\
    &ACM(A)-1000 &1000/1000/1000 &3790/4539/4159 &17 \\
    &DBLP(A)-500 &500/500/500 &1710/1788/1785 &17 \\
    &DBLP(A)-1000 &1000/1000/1000 &3627/3807/3807 &17 \\ \bottomrule
    \end{tabular}
    \caption{Dataset summary.}
    \label{tab:dataset}
\end{table*}

\paragraph{Baseline methods descriptions.} A brief overview of baseline methods is given in Table~\ref{tab:base}. We categorize baseline methods based on three properties: method type (consistency-based, embedding-based and OT-based), alignment scenarios (plain and attributed networks), setting (pairwise and multi-network alignment) and whether it is a hierarchical approach or not.

\paragraph{Dataset descriptions.} The datasets used in our experiments include:
\begin{itemize}
    \item{ER:} Synthetic Erd\"{o}s-R\'{e}nyi random graphs. Each network is a permutation of the base network with noise added by first inserting 10\% edges and then removing 15\% edges.
    \item{Douban~\cite{zhong2012comsoc}:} Social networks modeling both online and offline activities. Nodes represent users and edges model friendships between two users. The original dataset contains 50k users and 5M edges. 
    \item{ACM~\cite{tang2008arnetminer}:} Co-authorship network of ACM Digital library. Nodes represent authors and an edge exists between two authors if they are co-author for at least one publication. Node attributes indicate the number of papers published in different venues by the node/author. The original dataset includes 9,916 nodes and 44,808 edges.
    \item{DBLP~\cite{tang2008arnetminer}:} Co-authorship network of DBLP bibliography. Nodes represent authors and an edge exists between two authors if they are co-author for at least one publication. Node attributes indicate the number of papers published in different venues by the node/author. The original dataset includes 9,872 nodes and 39,561 edges.
\end{itemize}
Owing to the high time and space complexities of many baseline methods, we randomly sample connected subgraphs with 500 and 1000 nodes from the ACM and DBLP dataset as the benchmark dataset. Detailed dataset statistics can be found in Table~\ref{tab:dataset}.

\paragraph{Machine configuration and code.} The proposed method is implemented in Python. Experiments are conducted on an Apple M1 CPU with 16 GB RAM and a NVIDIA Tesla V100 SXM2 GPU. We will release the source code and the datasets after the paper is published.

\section{Impacts, Limitations and Future Works}
In this paper, we study the multi-network alignment problem from the view of multi-marginal optimal transport. This work could benefit various downstream applications, including computer vision~\cite{swoboda2019convex}, high-order recommendation~\cite{man2016predict} and drug discovery~\cite{chen2016fascinate}. We emphasize that our work has no negative social impact and discuss the limitations and possible future directions as follows:

\begin{itemize}
    \item\verb|Multi-level alignment|: the current work only considers two level: node-level and cluster-level. A possible future direction is to generalize the current framework to multiple levels by recursively applying the FGW barycenter to form aligned clusters at multiple levels, through which the time and space complexities can be further reduced dramatically. 
    \item\verb|Leveraging supervision|: In this paper, we incorporate the supervision, i.e., anchor node sets, by generating a unified positional embedding based on RWR. Though quite effective, we can still leverage supervision from other aspects. For example, enforcing the one-hot alignment scores for anchor node sets based on the supervised optimal transport~\cite{cang2022supervised}, or adding alignment preference regularization~\cite{zeng2023parrot} to penalize the inconsistency between alignment results and supervision.
    \item\verb|Accelerating with low-rank approximation|: Although the proposed \name\ achieves an exponential reduction in both time and space complexities via the hierarchical approach, but the complexity is still exponential w.r.t. number of graphs. We may reduce the time and space complexities by exploring the low-rank property of graph data based on the low-rank Sinkhorn factorization~\cite{scetbon2021low}.
\end{itemize}